\documentclass{article}

\usepackage{microtype}
\usepackage{graphicx}
\usepackage{subfigure}
\usepackage{booktabs} %

\usepackage{hyperref}

\usepackage[accepted]{icml2023}

\usepackage{amsmath}
\usepackage{amssymb}
\usepackage{mathtools}
\usepackage{amsthm}
\usepackage{nccmath}

\usepackage[capitalize,noabbrev]{cleveref}

\usepackage{bm}
\usepackage{enumitem}
\usepackage[thicklines]{cancel}
\usepackage{thm-restate}
\usepackage{inconsolata}
\usepackage{lipsum}

\usepackage[disable,textsize=tiny]{todonotes}
\makeatletter
\newcommand*\iftodonotes{\if@todonotes@disabled\expandafter\@secondoftwo\else\expandafter\@firstoftwo\fi}  %
\makeatother

\newcommand{\justification}[1]{\text{{\color{gray}(#1)}}}

\usepackage{macros}

\icmltitlerunning{Principled Gradient-based Markov Chain Monte Carlo for Text Generation}

\begin{document}

\twocolumn[
\icmltitle{Principled Gradient-based Markov Chain Monte Carlo for Text Generation}

\icmlsetsymbol{equal}{*}

\begin{icmlauthorlist}
\icmlauthor{Li Du}{jhu,ethz}
\icmlauthor{Afra Amini}{ethz}
\icmlauthor{Lucas Torroba Hennigen}{mit}
\icmlauthor{Xinyan Velocity Yu}{usc} \\
\icmlauthor{Jason Eisner}{jhu}
\icmlauthor{Holden Lee}{jhu}
\icmlauthor{Ryan Cotterell}{ethz}
\end{icmlauthorlist}

\icmlaffiliation{jhu}{Johns Hopkins University}
\icmlaffiliation{ethz}{ETH Z\"{u}rich}
\icmlaffiliation{mit}{MIT CSAIL}
\icmlaffiliation{usc}{University of Southern California}

\icmlcorrespondingauthor{Li Du}{leodu@cs.jhu.edu}

\vskip 0.3in
]

\printAffiliationsAndNotice{}

\begin{abstract}

Recent papers have demonstrated the possibility of energy-based text generation by adapting gradient-based sampling algorithms, a paradigm of MCMC algorithms that promises fast convergence.
However, as we show in this paper, previous attempts on this approach to text generation all fail to sample correctly from the target language model distributions.
To address this limitation, we consider the problem of designing text samplers that are faithful, meaning that they have the target text distribution as its limiting distribution.
We propose several faithful gradient-based sampling algorithms to sample from the target energy-based text distribution correctly, and study their theoretical properties.
Through experiments on various forms of text generation, we demonstrate that faithful samplers are able to generate more fluent text while adhering to the control objectives better.

\end{abstract}

\section{Introduction}
\label{sec:intro}

Recent papers have demonstrated the possibility of performing controlled text generation from pretrained language models by formulating energy-based models over text and applying Markov Chain Monte Carlo (MCMC) algorithms~\citep{qin2022cold,kumar-2022,mireshghallah-etal-2022-mix,amini2023structured}.
Such an energy-based formulation offers a considerable degree of versatility, allowing a generic pretrained language model to be coupled with arbitrary energy functions that express desired traits for the output text.
As is common in energy-based models (EBMs), one can then use MCMC algorithms to draw samples from these complicated distributions where the normalization constant is intractable to compute \citep{lin2021limitations}.

However, the discrete nature of text-based EBMs and their underlying combinatorial state space make it challenging to sample from them in a reasonable amount of time~\citep{Deng2020Residual}.
To address this problem, existing approaches commonly exploit the fact that a language model, as well as the auxiliary energy functions, are all differentiable with respect to the embedding space.
This observation allows the MCMC procedure to make use of first-order gradient information, potentially accelerating convergence.

As one of the most successful gradient-based samplers, Hamiltonian Monte Carlo (HMC) and its variants \citep{duane1987,neal1993probabilistic,hoffman2014nuts} have been proven to be highly effective in sampling from high-dimensional, continuous distributions, making them the default choice of sampler in many probabilistic programming languages \cite{carpenter2017stan,bingham2018pyro,phan2019numpyro}. Adapting HMC into a discrete setting, \citet{amini2023structured} recently proposed a promising sampler for controlled text generation.
Alternatively, Langevin dynamics \cite{grenander1994langevin,welling2011langevin}, another gradient-based sampler, has been a more popular candidate to adapt into NLP models due to its simplicity.\footnote{In fact, it is well-known that Langevin dynamics can be seen as HMC where the Hamiltonian dynamics are simulated for a single step. 
See, e.g., \citet{neal1993probabilistic} or \citet{kennedy1990}.}
As a result, \citet{qin2022cold} and \citet{kumar-2022} proposed text samplers inspired by Langevin dynamics.\looseness=-1

Unfortunately, a closer look reveals that \textit{none} of these gradient-based Markov chains provably converge to their intended distributions in the limit, as we theoretically and empirically show in \cref{sec:faithful-text-samplers}.
This observation seemingly contradicts the empirical results reported in these papers and begs the question: What would happen if we sample from the target distribution correctly?

In this work, we tackle this question by proposing several tractable gradient-based samplers that are \emph{faithful} to the target energy-based text distribution, meaning that they have the correct limiting distribution.
We derive two novel samplers, based on Langevin dynamics and the Gibbs sampler, respectively, and then develop their adaptive and hybrid variants.
When applicable, we will also prove convergence and mixing properties of our proposed samplers.
It is not inherently true that faithful samplers should outperform unfaithful methods.
Indeed, prior methods may have been optimized for useful inductive biases that work to their advantage and boost performance \emph{despite} their not sampling from their intended distributions.
Thus, through experiments on various forms of text generation, we explore whether our proposed faithful samplers can generate more fluent text while adhering to the control target better.
Our experimental results suggest that faithful samplers do, in general, outperform unfaithful samplers.\looseness=-1

\section{Energy-based Models of Text}
\label{sec:ebm-text}

Pretrained language models \citep{radford2019language,raffel2020t5,brown2020gpt3} have demonstrated impressive abilities in generating fluent texts.
They do so by factorizing a string-valued distribution $\plm(\vw)$ ($\vw\in\alphabet^*$) over some vocabulary $\alphabet$ with local normalization \citep{du2023measure}
\begin{align}
    \label{eq:auto-regressive-lm}
    \plm(\vw)=\plm(\eos\mid\vw)\prod_{n=1}^{N} \plm(w_n\mid \vw_{<n})
\end{align}
and train the local conditionals $\plm(\cdot\mid\vw_{<n})$ on massive amount of text.
Since such texts often come from heterogeneous sources (e.g., newspapers, blog posts, etc.) and does not fit into any particular category or style, we say that such a distribution is \textit{unconstrained}.

On the other hand, it is often useful to perform \textit{controlled generation}---sampling a text that satisfies one or several soft constraints.
Such constraints might be lexical, semantic, grammatical, or arbitrary functions that evaluate some global property over the entire sequence.  
When we reduce the (log-)probabilities of each text to the degree it violates the constraints, we are left with the challenging problem of sampling from a (discrete) unnormalized probability distribution, which may be presented in the form of an energy-based model (EBM; \citealt{hinton2002,lecun2006ebm}):
\begin{align}
    \label{eq:ebm}
    \pi(\vw)=\frac{1}{Z}\exp(-U(\vw))
\end{align}
Here $U(\vw)$ is called the \textit{energy function}.\footnote{The notation $U(\vw)$ rather than the usual $E(\vx)$ is drawn from the HMC literature, which calls it the \textit{potential function}.} %
The flexibility of this framework lies in the fact that one can refine an existing model by coupling its energy function with arbitrary functions that express the desired attributes of the output text.
Concretely, we can set 
\begin{align}
    \label{eq:general-control-text-ebm}
    U(\vw)=\ulm(\vw) + \sum_{i=1}^I U_i(\vw)
\end{align}%
where $\ulm(\vw)\defeq-\log \plm(\vw)$ (from \cref{eq:auto-regressive-lm}) and each $U_i(\vw)$ measures the extent to which the sequence $\vw$ satisfies the $i$\textsuperscript{th} constraint. This energy function yields a distribution that is related to $\plm(\vw)$ but places more probability mass on the sequences that better satisfy the constraints. %

\section{Text Generation as MCMC}

\subsection{Sampling from EBMs}
\label{sec:sampling-from-ebms}

The flexible formulation in \cref{eq:ebm,eq:general-control-text-ebm} allows us to cast controlled text generation as the problem of sampling from an energy-based model. However, EBMs can be challenging to sample from.

Consider sampling a sequence of $N$ words $\vw=w_1 \cdots w_N\in\alphabet^N$ from the EBM defined by \cref{eq:ebm,eq:general-control-text-ebm}.
The normalization constant $Z$ from the EBM defined by \crefrange{eq:ebm}{eq:general-control-text-ebm} is then an intractable sum of $|\alphabet|^N$ terms.\footnote{It can be tractable in special cases such as linear-chain graphical models, but not in general.}
Similarly, the locally normalized conditional probabilities needed for left-to-right autoregressive sampling---which are effectively ratios of normalization constants---are also intractable \citep{lin2021limitations}. 

As in other situations where exact sampling is unavailable,  we may resort to Markov Chain Monte Carlo (MCMC), an approach to sampling approximately from unnormalized distributions \citep{metropolis1953equation}. 
In our situation, the combinatorially large underlying state space $\alphabet^N$ means that naive MCMC algorithms such as the Random Walk Metropolis (RWM) would have near-zero acceptance rate.
Gibbs sampling, another commonly used MCMC algorithm, requires one to be able to efficiently sample from the conditional $\pi(w_n\mid \vw_{\setminus n})$.\footnote{We use $\vw_{\setminus n}$ to denote the set of random variables of all indices \textit{except} $i$, i.e., $\vw_{\setminus n}=w_1\cdots w_{n-1} w_{n+1} \cdots w_N$.}
This is also impractical since evaluating $\pi(\cdot\mid\vw_{\setminus n})$ in a locally normalized LM would again involve a summation of $|\alphabet|$ terms.

\subsection{Gradient-based Sampling through Relaxation}

The challenges outlined in the previous section indicates that we need additional techniques to obtain a sampling procedure that yields quality samples in a reasonable amount of time.
Observing that $\ulm$ defined from a pretrained neural LM is differentiable, as well as possibly the constraint functions $U_i$, several prior works have taken inspiration from the success of gradient-based sampling in other domains \citep{neal2011hmc,hoffman2014nuts,carpenter2017stan,welling2011langevin,du2019implicit,song2020score-based} and attempted to leverage gradient information when sampling from text-based EBMs \cite{qin2022cold,kumar-2022,amini2023structured}. 

However, problems arise when gradient-based sampling algorithms such as HMC or Langevin dynamics only directly apply to continuous distributions. Therefore, to apply such algorithms to sample from discrete distributions, prior works that developed gradient-based sampling for energy-based text generation all focus on continuous relaxations of the underlying discrete space. %
In particular, \citet{qin2022cold} allows the discrete EBM to assume inputs in the entire continuous $\R^d$ which will lead to the language model taking input vectors that do not correspond to any word embeddings;
\citet{kumar-2022} on the other hand allow the sample trajectory to traverse outside of the discrete word embedding space but eventually project them back;
\citet{amini2023structured} uses Voronoi tessellation to relax the discrete distribution over word embeddings into a piecewise continuous distribution with the embeddings as the centers of the Voronoi cells.
Unfortunately, none of these continuous relaxation techniques resulted in a sampler that can correctly sample from their target energy-based distribution over text.

\subsection{Faithfulness of Gradient-based Text Samplers}
\label{sec:faithful-text-samplers}

In this section, we explain and illustrate in detail why existing methods fail to converge to their intended distributions and thus are unfaithful samplers. To do so, we consider the setting of sampling a sequence of $N$ words $\vw=w_1\cdots w_N\in\alphabet^N$ from an energy-based sequence model.
We denote the corresponding word embeddings $\vx=(\vx_1, \dots,\vx_N)\in \calX\defeq \calV^{N}\subset\R^{Nd}$ where $\calV\subset\R^d$ is the discrete set of word embeddings.

\paragraph{\cold \textmd{\citep{qin2022cold}}.} \cold observes that, while the EBM induced from a language model is defined as
\begin{align}
    \label{eq:lm-ebm}
    \pilm(\vx)=\frac{\exp(-\ulm(\vx))}{\sum_{\vy\in\calX} \exp(-\ulm(\vy))}~,~\vx\in\calX,
\end{align}
where $\ulm$ as a function can also take vectors other than the word embeddings as its input.
\cold proceeds to use Langevin dynamics that include $U(\vx)$ as an energy function over the continuously relaxed space.\footnote{In practice, \cold operates in logits space and uses a weighted average of word embeddings, which corresponds to the convex cone of the word embedding space. However, this detail does not affect our following point.} 
In effect, \cold is sampling from a \textit{density} similar to the following
\begin{align}
    \label{eq:cold-continuous-relax}
    \pcoldlike(\vx) = \frac{\exp(-\ulm(\vx))}{\displaystyle \int_{\R^{nd}} \exp(-\ulm(\vy)) \mathrm{d}\vy}~,~ \vx\in\R^{Nd}.
\end{align}
Even though \cref{eq:lm-ebm} and \cref{eq:cold-continuous-relax} superficially have the same numerator, the two distributions can have drastically different features or may even be unrelated.\footnote{In fact, if the integral in \cref{eq:cold-continuous-relax} does not converge, the sampling process can escape to infinity (i.e., the measure is not well-defined).}
This means that, when \cold performs Langevin dynamics over \cref{eq:cold-continuous-relax}, its samples cannot be regarded as from \cref{eq:lm-ebm}.
We will illustrate this further in \cref{ex:toy-lm}.

\paragraph{\mucola \textmd{\citep{kumar-2022}}.} Similar to \cold, \mucola also takes Langevin steps in the underlying continuous space $\R^{Nd}$ but eventually projects back to the embedding space using Euclidean distance:
\begin{align}
    \label{eq:mucola-1}
    \vx'=\proj_{\calX}\left(
        \vx-\frac{\alpha}{2}\nabla U(\vx)+\sqrt{\alpha} \vxi
    \right)
\end{align}
where $\vxi\sim\calN(0,I)$ is the Gaussian noise vector. 
This procedure cannot sample from the continuous distribution in \cref{eq:lm-ebm}, because the gradients of $U$ at discrete points in $\calX$ do not determine the values of $U$ in $\calX$ even up to an additive constant, and hence do not uniquely determine the distribution. 
Again, see \cref{ex:toy-lm} for an illustration. %

\paragraph{\svs \textmd{\citep{amini2023structured}}.}
As mentioned earlier, the continuous relaxation resulted from \svs is piecewise continuous, with each continuous region being a Voronoi cell.
To be able to sample from the correct distribution, \svs requires high-dimensional Gaussian integral over the Voronoi cells. For reasons which we will detail in \cref{sec:high-dim-int}, this integral is unfortunately infeasible to compute.
As a result, \svs makes the assumption that all Voronoi cells in the word embedding space in pretrained LMs have equal Gaussian volume.
We note that the equal measure assumption in \svs \textit{can} be true in very specific circumstances, such as in binary discrete distributions, e.g., the Ising models.
Unfortunately, such an assumption is unrealistic for real-world language models.

\begin{figure}[t]
    \centering
    \includegraphics[width=0.9\columnwidth]{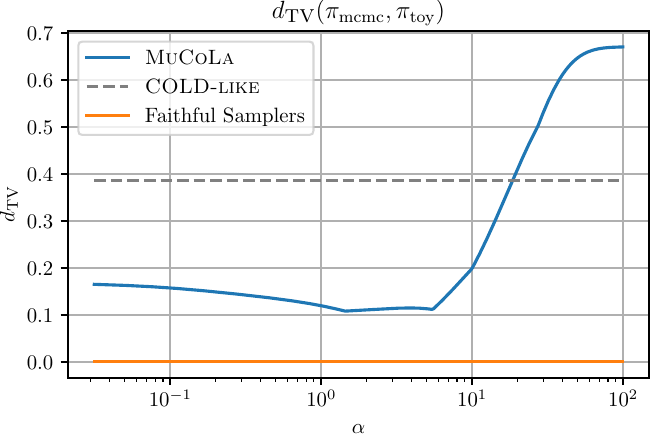}
    \caption{Total variation distance between $\pi_\mathrm{mcmc}$, the limiting distribution of MCMC algorithms from previous works, and $\pi_\toy$, the toy language model distribution from \cref{ex:toy-lm}.
    $\pi_\mathrm{mcmc}$ is computed with spectral decomposition when possible.
    We can observe that the limiting distribution of \cold is far from the target distribution, and \mucola, depending on its step size $\alpha$, may be close to the target distribution. Nevertheless, it does not have the correct distribution for any $\alpha$.
    }
    \label{fig:tv-distances}
\end{figure}
\begin{myexample}[A Toy Energy-based LM]
\label{ex:toy-lm}
To further illustrate the previous claims, we consider a toy energy-based LM over a sequence of $N$ tokens, with a binary vocabulary and a one-dimensional embedding $\calV=\alphabet=\{-1,+1\}$. The energy function we use has the following form
\begin{align}
    U(\vx)=-\beta(\textstyle\frac{1}{2}\vx^\top A \vx+\vb^\top\vx).
\end{align}
with $\vx\in\calV^N$ and $\pi_\toy(\vx)\propto\exp(-U(\vx))$.
Concretely, we set $A$ to be the adjacency matrix of an $N$-cycle and $\vb=\bm{0}$.
In this setting, our energy-based LM in fact corresponds to a linear-chain Ising model with zero magnetic field.

We choose this model for the following reasons:\looseness=0
\begin{enumerate}
    \item The energy function is differentiable, and hence all previous algorithms apply;
    \item When $N$ is not too large, we can compute the exact distribution;
    \item The binary vocabulary allows us to compute the transition matrix of \mucola exactly as well as its stationary distribution.\footnote{The transition probability of \mucola is in general infeasible to compute. See \cref{sec:high-dim-int-in-mcmc}.}
\end{enumerate}
We use spectral decomposition of the transition matrix to calculate the exact stationary distribution of \mucola.
For \cold, we estimate the multi-dimensional Gaussian's quadrant probabilities with 1 million samples.

From \cref{fig:tv-distances}, we can see that \cold has a very different distribution compared to the toy language model distribution $\pi_\toy$, as we remarked earlier; 
On the other hand, we interestingly observe that, for a certain range of $\alpha$, \mucola can in fact approximate the true distribution fairly well. This may explain the fact that \mucola performs better than \cold in actual language generation tasks. Nevertheless, \mucola is not able to sample from the true distribution regardless of the value of $\alpha$.\footnote{We also note that, in this specific model, \svs is able to sample from the correct distribution because the Voronoi cells induced by the embeddings have equal measure due to symmetry. However, this is not true in general language models.}
\end{myexample}

\section{Background: MCMC}

Before we investigate the question of designing faithful and efficient gradient-based samplers for text generation, we first review the foundational principles that underpin the validity of the MCMC procedure.

Markov Chain Monte Carlo (MCMC; \citealt{metropolis1953equation}) is based on the idea that to produce samples from a target distribution $\pi(x)$, one can design a transition kernel $p(x'\mid x)$ such that the resulting Markov chain has the target distribution as its limiting distribution.
In finite discrete spaces, such as sampling sentences up to a fixed length, one designs the MCMC transition kernel to satisfy the following two criteria to guarantee convergence to some intended distribution $\pi(x)$:

\begin{enumerate}[label=(C\arabic*)]
    \item\label{item:mc-req-ergodic} \textit{The chain is ergodic.}
    This means that, regardless of the starting state, the chain has a nonzero probability of being at every state after a sufficient number of steps.
    Ergodicity is equivalent to being irreducible and aperiodic.
    \item\label{item:mc-req-invariance} \textit{The target distribution is invariant under the transition kernel.} 
    This means that, if the chain starts with the target distribution, it will stay in the target distribution, i.e. 
    \begin{align}
        \pi(x)=\sum_{y} p(x \mid y)\pi(y).
    \end{align}
\end{enumerate}

The reason that the above two criteria guarantee convergence to the target distribution is very simple. First of all, all finite state Markov chains have at least one stationary distribution. 
Adding the ergodicity requirement \ref{item:mc-req-ergodic} guarantees that the chain has a unique stationary distribution and the chain converges to that distribution, and \ref{item:mc-req-invariance} ensures that the target distribution $\pi(x)$ \emph{is} this unique stationary distribution. Therefore, \ref{item:mc-req-ergodic} and \ref{item:mc-req-invariance} combined imply that the chain will always converge to the target distribution regardless of its starting state.

In practice, \ref{item:mc-req-invariance} is often proved by establishing the detailed balance equation
\begin{equation}
    \label{eq:detailed-balance}
    \pi(x)p(x'\mid x)=\pi(x')p(x\mid x')
\end{equation}
which implies that $\pi(x)$ is a stationary distribution of \mbox{$p(\cdot\mid\cdot)$}.
When \cref{eq:detailed-balance} holds for a given Markov chain $p(\cdot\mid\cdot)$, we also say that the chain is reversible with respect to distribution $\pi(\cdot)$ and $\pi(\cdot)$ a \textit{reversing distribution} for $p(\cdot\mid\cdot)$.

Algorithmically, detailed balance (\cref{eq:detailed-balance}) is often achieved by using the Metropolis--Hastings acceptance procedure \citep{metropolis1953equation,hastings1970}.

\paragraph{Metropolis--Hastings Acceptance.}
Metropolis--Hastings acceptance is a procedure to convert \textit{any} Markov kernel $q(\cdot\mid\cdot)$ over $\calX$, called a \textit{proposal distribution}, into one that has the target distribution as its stationary. In each iteration, it draws a sample $x'$ from $q(\cdot\mid x)$ and then \textit{accepts} $x'$ with the \textit{acceptance probability}
\begin{align}
    \label{eq:metropolis}
    \alpha(x'\mid x) = \min\left\{
        1, \frac{\pi(x')q(x\mid x')}{\pi(x)q(x'\mid x)}
    \right\}.
\end{align}
In the case $x'$ is rejected, the chain remains at $x$.
One easily checks that the chain derived from the acceptance procedural $p(x'\mid x)=\alpha(x'\mid x)q(x'\mid x)$ is a reversible chain with $\pi(\cdot)$ as its reversing distribution.

In this work, unless otherwise stated, all our algorithms are corrected with Metropolis--Hastings and hence we only need to specify the proposal distribution $q(\cdot\mid\cdot)$.
However, we feel important to point out that Metropolis--Hastings isn't always necessary.
For example, by sampling from the true conditional, Gibbs sampling has a constant acceptance probability of 1, and hence the Metropolis--Hastings step can be omitted.
One may alternatively design an irreversible Markov kernel that directly satisfies \ref{item:mc-req-invariance} without satisfying \cref{eq:detailed-balance} (see, e.g., \citealp{sohl-dickstein-2014,diaconis-2000}).\looseness=-1

\paragraph{Mixing Time.}
We wish to design MCMC algorithms that converge to the target distribution in a reasonable amount of time, and hence
another important property of a given Markov chain is how fast it converges to the stationary distribution.
This quantity is measured by the \textit{mixing time}, $t_\mix$. 
Denoting $\P_{\vx}^{t}$ as the $t^\text{th}$ step distribution of a Markov chain started at state $\vx$, the $\varepsilon$-mixing time is defined as
\begin{align}
    \label{eq:tmix-def}
    t_\mix(\varepsilon) = \inf\left\{ t: \sup_{\vx\in\calX} d_\tv(\P_{\vx}^{t}, \pi) \leq \varepsilon \right\}
\end{align}
where $d_\tv(\cdot,\cdot)$ is the total variation distance\footnote{Recall that the total variation distance is defined as $d_\tv(\mu,\nu)\defeq\sup_{E}|\mu(E)-\nu(E)|$.} and $\pi$ is the stationary distribution of the Markov chain.
In words, $t_\mix(\varepsilon)$ is the minimum amount of time necessary to reach within $\varepsilon$ distance to the stationary distribution regardless of the starting state.

\section{Faithful Gradient-based Text Generation}
In this section, we focus on developing faithful samplers.
We first develop a Langevin-based sampler in \cref{sec:pncg}, which we term \pncg and discuss its theoretical properties in \cref{sec:pncg-properties}.
We then develop a Gibbs-based sampler in \cref{sec:gwl}.
We conclude with a discussion on hybrid samplers in \cref{sec:hybrid-samplers}.

\subsection{A Langevin-based Sampler}
\label{sec:pncg}

In our preliminary experiments, we found that \mucola \citep{kumar-2022}, a Langevin-based sampler, is the best candidate due to its simplicity and ability to generate relatively fluent sentences from LM-based EBMs.
An obvious solution is to add Metropolis--Hastings correction to \mucola.
Unfortunately, metropolizing \mucola will lead to the same high-dimensional integral that made the HMC-based sampler in \citet{amini2023structured} infeasible.\footnote{For details of why the integral shows up in metropolitan \mucola, see \cref{sec:high-dim-int-in-mcmc} for details.}
However, notice that the key property of the \mucola update equation
\begin{restatable}{donothing}{mucolaUpdate}
\begin{subequations}
\label{eq:mucola-update-eqn}
\begin{align}
    \vxi&\sim\calN(\bm{0},I) \\
    \vx'&=\proj_{\calX}\left(\vx-\frac{\alpha}{2}\nabla U(\vx) + \sqrt{\alpha}~\vxi\right)
\end{align}
\end{subequations}
\end{restatable}

\noindent is that the word embeddings further away from $\vmu_{\vx}\defeq\vx-\frac{\alpha}{2}\nabla U(\vx)$ will have lower probability to be sampled.
This insight indicates that the projection operator, which introduces the infeasible integral in the Metropolis--Hastings correction, is not necessary. 
One can introduce a discrete proposal distribution $q(\vx'\mid\vx)$ for $\vx'\in\calX$ with the same property by computing the ``Gaussian score'' at every word embedding and then normalizing, i.e.
\begin{subequations}
\begin{align}
    q(\vx'\mid&\vx)\propto
        \exp\left(-\frac{\left\|\vx'-\vmu_{\vx}\right\|_2^2}{2\alpha} \right) \\
    &=\exp\left(
        -\frac{1}{2\alpha}
        \left\|
            \vx'-\left(\vx-\frac{\alpha}{2}\nabla U(\vx)\right)
        \right\|_2^2
    \right) \label{eq:pncg-form-1}
\end{align}
\end{subequations}
With a few steps of derivation (detailed in \cref{sec:pncg-derivation}), we can rewrite the proposal in \cref{eq:pncg-form-1} as
\begin{align}
    \label{eq:pncg-form-2}
    \resizebox{0.89\hsize}{!}{$
    q(\vx'\mid\vx)\propto\exp\bigg(
        -\underbrace{\frac{1}{2}\nabla U(\vx)^\top(\vx'-\vx)}_{\mytag{Term (A1)}{term:A1}}
        -\underbrace{\frac{1}{2\alpha}\|\vx'-\vx\|_2^2}_{\mytag{Term (A2)}{term:A2}}
    \bigg).
    $}
\end{align}
Let us examine \cref{eq:pncg-form-2} more closely. We notice that \ref{term:A1} is in effect doing a first-order Taylor expansion, i.e., $U(\vx')-U(\vx)\approx\nabla U(\vx)^\top (\vx'-\vx)$, in an attempt to move to a state with lower energy.
On the other hand, first-order approximation is only accurate locally, and hence \ref{term:A2} acts as a regularizer that decreases the probability of moving to $\vx'$ that is too far from $\vx$.

Finally, when applying \cref{eq:pncg-form-2} to realistic language models such as GPT-2, we found that the $\ell_2$-norm penalty often runs into pathological situations where a few indices' large deviation disrupts the proposal distribution and results in low acceptance rate. We hypothesize that this is due to the unusual geometry of the underlying embedding space \cite{mimno-thompson-2017-strange} and found that using alternative norms is an effective remedy. We now arrive at our final form of proposal distribution:
\begin{align}
    \label{eq:pncg-final-form}
    \resizebox{.89\hsize}{!}{$
    q(\vx'\mid\,\vx) \propto\exp\bigg(
        -\underbrace{\frac{1}{2}\nabla U(\vx)^\top(\vx'-\vx)}_{\mytag{Term (B1)}{term:B1}} 
        -\underbrace{\frac{1}{2\alpha}\|\vx'-\vx\|_p^p}_{\mytag{Term (B2)}{term:B2}}
    \bigg).
    $}
\end{align}
We call this method $\ell_p$-{N}orm Constrained Gradient sampler (\textbf{$p$-NCG}), due to its connection to the Norm Constrained Gradient sampler proposed in \citet{rhodes2022enhanced}.%

\subsection{Properties of $p$-NCG}
\label{sec:pncg-properties}

\paragraph{Independence of Positions.}
Suppose we are sampling a sequence of length $N$ using the word embeddings: $\vx=[\vx_1~\dotsm~\vx_N]\in\R^{Nh}$ where each $\vx_n\in\calX\subset\R^h$ is a word embedding.
The proposal in \cref{eq:pncg-final-form} factorizes as a product that involves each individual word embedding:
\begin{align}
    \prod_{n=1}^N \exp\left(
        -\frac{1}{2}\nabla_n U(\vx)^{\top} (\vx'_n-\vx_n)
        -\frac{1}{2\alpha} \|\vx'_n-\vx_n\|_p^p
    \right).
\end{align}
This means that the proposal treats each word position conditionally independent of the other positions given the current sequence. This allows us to sample each word embedding in parallel from this proposal.

\paragraph{Convergence Analysis.} Another interesting property of the \pncg proposal is that when used unadjusted\footnote{As is standard in MCMC literature, we say that a proposal is used \textit{unadjusted} if we omit the Metropolist-Hastings correction and accept every sample.} on discrete log-quadratic distributions, such as the Ising models, its stationary distribution converges to the target distribution when the step size tends to zero. We make this precise below.
\begin{definition}
    \label{def:log-quad}
    Let $\pi(\vx)$ be a discrete distribution over $\calX\subset\R^d$ where $|\calX|<\infty$. $\pi$ is \emph{log-quadratic} if it can be expressed as
    \begin{align}
        \pi(\vx)\propto \exp\left( \vx^\top A \vx + \vb^\top \vx \right)
    \end{align}
    for some $A\in\R^{d\times d}$ and $\vb\in\R^d$.
\end{definition}

\begin{restatable}{theorem}{pncgConvergence}\label{thm:pncg-convergence}
    Let $\pi(\vx)$ be a discrete log-quadratic distribution as defined in \cref{def:log-quad}.
    For any $\alpha>0$, there exists a unique distribution $\pi_\alpha(\vx)$ such that the Markov chain defined by $q$ in \cref{eq:pncg-final-form} is reversible with respect to $\pi_\alpha$. Further, $\pi_\alpha \to \pi$ weakly as $\alpha\to 0$.
\end{restatable}
\begin{proof}[Proof Idea]
    The key insight of the proof is that first-order approximation of a quadratic energy function will leave a symmetric second-order error term. One can exploit this symmetry to construct a reversing distribution  and show that it converges to the target distribution. See \cref{sec:pncg-convergence-proof} for the full proof.
\end{proof}

\paragraph{Mixing-time Analysis.} When unadjusted proposals exhibit limiting behaviors as in \cref{thm:pncg-convergence}, it is tempting to use the proposal without using Metropolis--Hastings correction, as argued in \citet{zhang2022}. However, as \cref{thm:pncg-mixing-time} shows, the mixing time increases exponentially as the step size decreases towards 0. This means that, in practice, using the unadjusted proposal with a small step size is infeasible.

\begin{restatable}{theorem}{pncgMixing}\label{thm:pncg-mixing-time}
    Let $\pi(\vx)$ be a discrete log-quadratic distribution as defined in \cref{def:log-quad}. There exists constants $c_1, c_2, Z>0$ that depends only on $\pi(\vx)$ such that the mixing time of $q$ satisfies
    \begin{align}
        t_{\mix}(\varepsilon) \geq \left(\frac{c_1}{Z}\exp\left(\frac{c_2}{2\alpha}\right) - 1 \right) \log\left(\frac{1}{2\varepsilon}\right).
    \end{align}
\end{restatable}
\begin{proof}[Proof Idea]
    We use \gershgorin disc theorem (\cref{thm:gershgorin}) to bound the location of the eigenvalues and then relate it to mixing time through a well-known inequality (\cref{thm:mixing-spectral}).
    See \cref{sec:pncg-mixing-time} for the full proof.
\end{proof}

\subsection{A Gibbs-based Sampler}
\label{sec:gwl}

In this section, we consider adapting the Gibbs sampler \citep{geman1984gibbs}.
Again, consider sampling a sequence of length $N$ with word embeddings $\vx=[\vx_1 \dotsm \vx_N]\in\R^{Nh}$ where 
 each $\vx_n\in\calX\subset\R^h$ is a word embedding. 
To be able to use Gibbs sampling, we need to be able to efficiently compute the conditional probabilities $\pi(\vx_n\mid\vx_{\setminus n})$, which is infeasible as we argued in \cref{sec:sampling-from-ebms}.

However, we recall the fact that Gibbs sampling is just a special case of Metropolis--Hastings, where the use of exact conditional $\pi(\vx_n\mid\vx_{\setminus n})$ results in an acceptance probability of 1.
We may therefore use an approximation of $\pi(\vx_n\mid\vx_{\setminus n})$ and correct for the approximation error with Metropolis--Hastings.
Specifically, we approximate $\pi(\vx_n\mid\vx_{\setminus n})$ by estimating the energy difference with Taylor expansion:
\begin{align}
    U(\dotsm,\hat{\vx}_n,\dotsm)-U(\dotsm,{\vx}_n,\dotsm)\approx\nabla_n U(\vx)^\top (\hat{\vx}_n-\vx_n)
\end{align}
and then sample from
\begin{align}
    \exp(-\nabla_n U(\vx)^\top (\vx_n'-\vx_n)).
\end{align}
However, using the first-order approximation directly will lead to a near-zero acceptance rate due to the fact that local approximations have extremely high errors when used over the entire word embedding space. 
We therefore need to restrict the proposal move locally, which we again achieve by adding a $p$-norm penalty to our proposal.
This yields a Gibbs-based proposal
\begin{align}
    &q(\vx_n'\mid\vx_{\setminus n}) \propto \nonumber \\ &\enskip\exp\left(
        -\nabla_n U(\vx)^\top (\vx_n'-\vx_n)-\frac{1}{\alpha}\|\vx_n'-\vx_n\|_p^p
    \right). \label{eq:gwl-proposal}
\end{align}

An important caveat is that, since we are already using Metropolis--Hastings correction, it is a waste of computation to have self-transition probabilities in the proposal distribution.\footnote{For example, the Metropolis sampler never proposes self-transitions, which is part of the reason for why it is known to mix faster than the standard Gibbs sampler (Glauber dynamics) on Ising model \citep[\S31.1, p.403]{mackay2003} or other binary distributions \citep{newman1999monte}.}%
This led us to remove the self-transition probability and arrive at our final form of the Gibbs-based proposal
\begin{align}
    \label{eq:gwl-no-self-trans}
    q(\vx_n'\mid\vx_{\setminus n}) \propto 
    &\begin{cases}
        0 & \text{when $\vx_n=\vx_n'$} \\
        \textrm{\cref{eq:gwl-proposal}} & \text{otherwise}
    \end{cases}.
\end{align}
Notice that \cref{eq:gwl-proposal} resembles \cref{eq:pncg-final-form} except for the factor $1/2$ and the single word update. For this reason, we call this sampler \textit{Gibbs with Langevin} (GwL). 

\paragraph{Scan Ordering.}
As with other Gibbs samplers, the scan ordering (the order in which each index is sampled) can greatly impact the sampler's efficiency \citep{he2016scan-order}.\footnote{This is despite the fact that systematic scan and random scan have long been conjectured to have similar mixing times up to logarithmic factors \citep[\S26, Open Question 3]{levin2017}.} 
In light of this, we experiment with both systematic scan as well as random scan when using GwL.

\subsection{Hybrid Samplers}
\label{sec:hybrid-samplers}

One naturally wonders why would one use GwL when \pncg can update multiple words at a time.
We found through experiments that, in the beginning, when the sequence is randomly initialized, \pncg indeed proposes to change multiple indices at once and can have a reasonably high acceptance rate.
However, once the chain is close to convergence and the sentence structure starts to emerge, \pncg only proposes to change at most 1 index at a time and proposes self-transition for roughly 15\% of the time.
For this reason, GwL, which never proposes self-transitions, can have higher statistical efficiency in the later stages of the sampling process.
In practice, we implement a hybrid sampler, where we use \pncg during the initial phase of the sampler and switch to GwL once the chain starts to converge.\looseness=-1

\begin{figure}
    \centering
    \includegraphics[width=0.9\columnwidth]{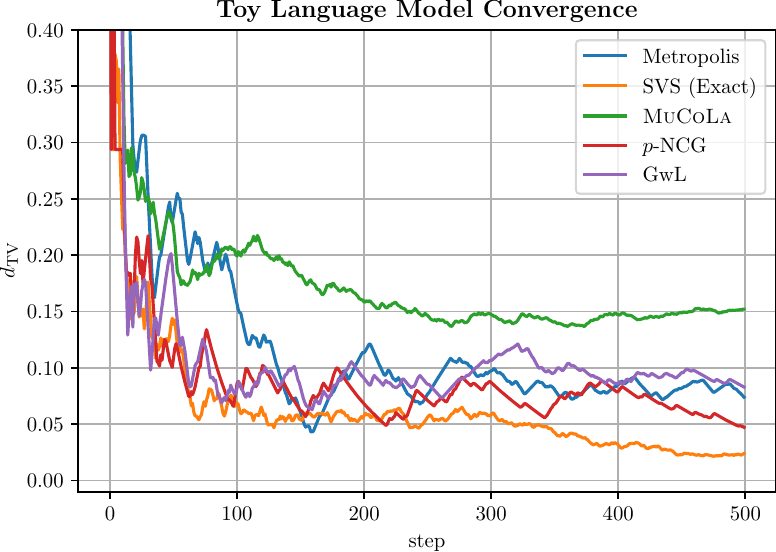}
    \caption{Total variation distance between the empirical distribution of different samplers (at different steps) and $\pi_\toy$, the true distribution of the toy language model from \cref{ex:toy-lm}.}
    \label{fig:toy-lm-convergence}
\end{figure}

\begin{figure}
    \centering
    \includegraphics[width=0.9\columnwidth]{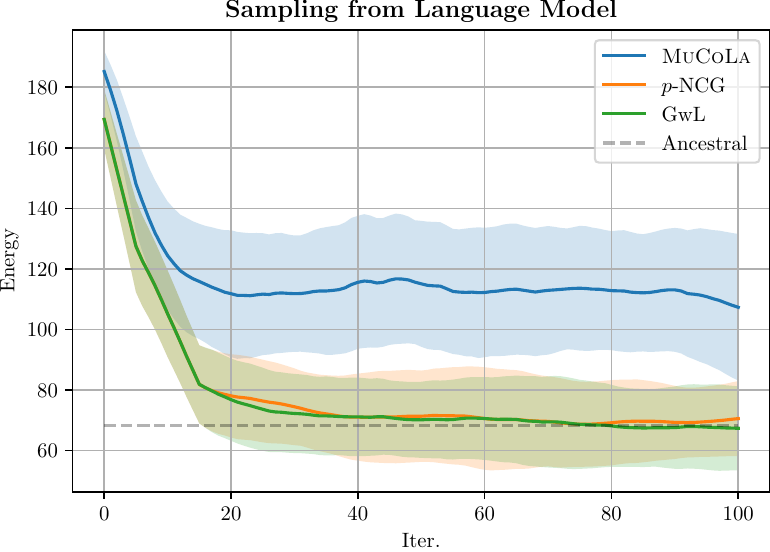}
    \caption{Energy traces of different samplers when sampling from GPT-2. 
    We observe that the faithful samplers (\pncg and GwL) converges to unbiased estimate of energy (estimated using ancestral sampling).
    On the other hand, the energy of the \mucola chain drops initially but suffers from systematic bias and is unable to converge to the true energy distribution, similar to the conclusion in our exact analysis in \cref{ex:toy-lm}.
    }
    \label{fig:lm-convergence}
\end{figure}

\begin{table*}[t!]
\centering 
\begin{tabular}{@{}lccccc@{}}\toprule
&  Success$(\uparrow)$ & PPL$(\downarrow)$ & Distinct-1$(\uparrow)$ & Distinct-2$(\uparrow)$ & Distinct-3$(\uparrow)$ \\ \midrule
{GPT-2} & $0.12 \pm 0.10$ & $5.10 \pm 2.06$ & $0.40$ & $0.56$ & $0.67$ \\
\fudge & $0.30 \pm 0.12 $ & $5.59 \pm 0.60$ & $0.39$ & $0.55$ & $0.65$\\ 
\mucola & $0.58 \pm 0.23$ & $33.09 \pm 36.32$ & $0.26$ & $0.4$ & $0.51$\\
\midrule
\svs-\langevin & $0.91 \pm 0.12$ & $14.26 \pm 2.55$ & $0.24$ & $0.39$ & $0.51$\\ 
\svs & $0.92 \pm 0.05 $ & $13.9 \pm 2.04$ & $0.22$ & $0.37$ & $0.49$\\
\midrule
$p$-NCG & $0.96 \pm 0.03 $ & $6.82 \pm 0.47$ & $0.23$ & $0.52$ & $0.78$ \\
$p$-NCG + GwL & $0.99 \pm 0.02$ & $5.17 \pm 0.38$ & $0.20$ & $0.44$ & $0.68$ \\
\bottomrule
\end{tabular}
\caption{Evaluation of different sampling methods on controlled generation, using three criteria: Success in following the control target determined by an external classifier (main metric), fluency (measured by perplexity), and diversity (measured by Distinct-$n$).}
\label{tab:control-results-avg}
\end{table*}

\section{Experiments}
\label{sec:experiments}

We now empirically assess the performance of our proposed samplers in a series of experiments.
The hyperparameter settings for all the experiments can be found in \cref{sec:experimental-setup}.

\subsection{Toy Example}
\label{sec:toy-lm-experiment}

We first apply different sampling methods to the toy language model discussed in \cref{ex:toy-lm}. Since we can compute $\pi_\toy$ exactly, we can compare the empirical distribution of the Markov chain up to a certain step to the true distribution by measuring the total variation distance between the two.\looseness=-1

We compare two of our proposed samplers, \pncg, and GwL, to baselines in prior works, \mucola and \svs.
We also included the standard Metropolis sampler for comparison.
Since \svs uses Gaussian augmentation \citep[\S4]{amini2023structured}, the resulting Hamiltonian yields a set of differential equations that can be solved in closed form. 
We therefore integrate the Hamiltonian dynamics exactly instead of using leapfrog steps, similar to the setup in \citet{pakman-paninski-nips2013}.\footnote{This exact integration is possible but can be difficult to implement when using \svs on actual language models. However, this is only an implementation speed up and does not alter the limiting distribution of \svs, which is incorrect for actual language models.} 
If the algorithm has a step size parameter, we tune this parameter via grid search.

The results are shown in \cref{fig:toy-lm-convergence}. We observe that the HMC-based sampler \svs has the best overall performance, due to its ability to traverse a long distance in the underlying space in a single step while preserving a perfect acceptance rate since its Hamiltonian is integrated exactly.
On the other hand, \pncg, GwL and Metropolis all have similar performance, partially due to the fact that the toy language model is too small in size to compare these samplers. Still, we can observe that all samplers except for \mucola are able to converge to the correct limiting distribution $\pi_\toy$, albeit at different rates.
Finally, we note that \mucola displays the systematic bias that we saw in \cref{ex:toy-lm}, where we calculated its stationary distribution exactly through spectral decomposition.
We see that \mucola's empirical distribution plateaus at a certain distance away from the true distribution.

\subsection{Sampling from Language Models}

Next, we apply our methods to sample from a language model.\footnote{We use the \texttt{GPT-2} checkpoint from the Huggingface library.}
As opposed to the experiment in \cref{sec:toy-lm-experiment}, the exact target distribution is not tractable.
We therefore use the energy distribution as a surrogate to measure whether the chain has converged.

In \cref{fig:lm-convergence}, we show the energy trace plots of different chains, and compare these against the mean energy of the target language model distribution. The mean energy shown is estimated with 2000 unbiased ancestral samples.
We observe that the faithful samplers, \pncg and GwL, quickly converge to the true energy distribution.
On the other hand, the energy of the \mucola chain decreases initially but is unable to converge to the true energy distribution since it is an unfaithful sampler.

\subsection{Controlled Generation}

Finally, we apply our methods to a controlled generation task.
We finetune GPT-2 on the E2E dataset \citep{novikova-etal-2017-e2e} which contains restaurant reviews for different types of food, e.g., Italian, Fast food, Japanese, etc.
We then task the model to generate a review of a specific food type $t\in\calT$.
To do so, we train classifiers to predict the food type $t$ from the input sequence $p_\cls(t\mid \vx)$ and use its log-likelihood of the food type as part of the energy function.
We implement the following baselines.

\paragraph{\fudge.} Introduced by \citet{yang-klein-2021-fudge}, \fudge samples tokens from the language model autoregressively, but weights the token probabilities at each position according to a classifier that determines whether the next token is likely to satisfy the constraint. In effect, by training classifiers to re-weight the per-step token probabilities under some global constraint, \fudge is distilling a globally-normalized EBM into a locally normalized one, which \citet{yang-klein-2021-fudge} aptly referred to as ``Future Discriminators''.

\paragraph{\mucola.} Introduced by \citet{kumar-2022}, \mucola forms a Markov chain using the update equation in \cref{eq:mucola-update-eqn} and defines the energy function as
\begin{align}
    U(\vx)=-\log\plm(\vx)-\beta \log p_\cls(t\mid \vx)
\end{align}
where $\beta$ is a hyperparameter intended to balance the classifier energy and the language model.

\paragraph{\svs and \svs-\langevin.} Introduced by \citet{amini2023structured}, both methods define a piecewise continuous distribution based on the Voronoi cells generated from the word embeddings. \svs-\langevin samples from this distribution using Langevin Dynamics, and \svs applies the appropriate form of HMC \citep{rhmc}.

\paragraph{Evaluation.} We sample 20 sentences of length 20 for each control target (resulting in a total of 140 sentences for each method). We evaluate the generations based on the following three metrics:
\begin{enumerate}
    \item \textbf{Success} is defined as the proportion of generations that followed the control target. To determine this, we train an external classifier (i.e., one that is different from the classifier that we use to generate the sequence).
    \item \textbf{Fluency} is measured by the perplexity under the language model.
    \item \textbf{Distinct-$n$} is a indicator of \textit{diversity}, which measures the ratio of unique $n$-grams in the set of generated samples.
\end{enumerate}

The results are shown in \cref{tab:control-results-avg}.
We also provide generated samples for each sampler in \cref{tab:control-sample} in \cref{sec:control-sample}.
We can see that, both \pncg and its hybrid variant with GwL almost always succeed in following the target while maintaining a high level of fluency. 
Notably, the hybrid sampler \pncg + GwL can maintain a level of fluency comparable to the unconditional language model while adhering to the control target.
This demonstrates that the faithful samplers are able to draw samples from the true distributions, which places high probability mass on the sequences with high fluency as well as better conformity with the constraints.
In contrast, \fudge obtains high fluency but often ignores the control target, while \svs and \svs-\langevin sacrifice fluency in exchange for better compliance with the control.

\section{Conclusions and Future Work}

In this work, we proposed two novel gradient-based samplers for generating text from energy-based models.
We analyzed and compared against previous works which we illustrated and proved to be unfaithful samplers, meaning that their limiting distribution is different from the text distribution they want to sample from.
We then demonstrated with experiments that faithful samplers have far better performance in realistic tasks on text generation in terms of both controllability as well as fluency.

We note that, while our work is a first step towards investigating principled probabilistic approaches of text generation, there are many possibilities in which our algorithms can be extended. For example, while we manually tune the step size $\alpha$ for each model, we may use automatic tuning methods introduced in \citet{hoffman2014nuts} that preserve detailed balance. We may also consider proposal merging algorithms used in \citet{horowitz1991,kennedy1991}.

\section*{Acknowledgements}

The experiments in this work was carried out at the Advanced Research Computing at Hopkins (ARCH) core facility, which is supported by the National Science Foundation (NSF) grant number OAC 1920103. 
Afra Amini is supported by the ETH AI Center doctoral fellowship.
We thank Justin T. Chiu for helpful discussion and comments.

\bibliography{custom}

\begin{thebibliography}{64}
\expandafter\ifx\csname natexlab\endcsname\relax\def\natexlab#1{#1}\fi

\bibitem[{Amini et~al.(2023)Amini, Du, and Cotterell}]{amini2023structured}
Afra Amini, Li~Du, and Ryan Cotterell. 2023.
\newblock \href {https://openreview.net/forum?id=vf77fTbgG3} {Structured voronoi sampling}.
\newblock In \emph{Thirty-seventh Conference on Neural Information Processing Systems}.

\bibitem[{Bingham et~al.(2018)Bingham, Chen, Jankowiak, Obermeyer, Pradhan, Karaletsos, Singh, Szerlip, Horsfall, and Goodman}]{bingham2018pyro}
Eli Bingham, Jonathan~P. Chen, Martin Jankowiak, Fritz Obermeyer, Neeraj Pradhan, Theofanis Karaletsos, Rohit Singh, Paul Szerlip, Paul Horsfall, and Noah~D. Goodman. 2018.
\newblock {Pyro}: Deep universal probabilistic programming.
\newblock \emph{Journal of Machine Learning Research}.

\bibitem[{Brown et~al.(2020)Brown, Mann, Ryder, Subbiah, Kaplan, Dhariwal, Neelakantan, Shyam, Sastry, Askell, Agarwal, Herbert-Voss, Krueger, Henighan, Child, Ramesh, Ziegler, Wu, Winter, Hesse, Chen, Sigler, Litwin, Gray, Chess, Clark, Berner, McCandlish, Radford, Sutskever, and Amodei}]{brown2020gpt3}
Tom Brown, Benjamin Mann, Nick Ryder, Melanie Subbiah, Jared~D Kaplan, Prafulla Dhariwal, Arvind Neelakantan, Pranav Shyam, Girish Sastry, Amanda Askell, Sandhini Agarwal, Ariel Herbert-Voss, Gretchen Krueger, Tom Henighan, Rewon Child, Aditya Ramesh, Daniel Ziegler, Jeffrey Wu, Clemens Winter, Chris Hesse, Mark Chen, Eric Sigler, Mateusz Litwin, Scott Gray, Benjamin Chess, Jack Clark, Christopher Berner, Sam McCandlish, Alec Radford, Ilya Sutskever, and Dario Amodei. 2020.
\newblock \href {https://proceedings.neurips.cc/paper_files/paper/2020/file/1457c0d6bfcb4967418bfb8ac142f64a-Paper.pdf} {Language models are few-shot learners}.
\newblock In \emph{Advances in Neural Information Processing Systems}, volume~33, pages 1877--1901. Curran Associates, Inc.

\bibitem[{Carpenter et~al.(2017)Carpenter, Gelman, Hoffman, Lee, Goodrich, Betancourt, Brubaker, Guo, Li, and Riddell}]{carpenter2017stan}
Bob Carpenter, Andrew Gelman, Matthew~D Hoffman, Daniel Lee, Ben Goodrich, Michael Betancourt, Marcus Brubaker, Jiqiang Guo, Peter Li, and Allen Riddell. 2017.
\newblock \href {https://www.jstatsoft.org/article/view/v076i01} {{Stan}: A probabilistic programming language}.
\newblock \emph{Journal of statistical software}, 76(1).

\bibitem[{Chalkis et~al.(2021)Chalkis, Fisikopoulos, Tsigaridas, and Zafeiropoulos}]{chalkis2021metabolic}
Apostolos Chalkis, Vissarion Fisikopoulos, Elias~P. Tsigaridas, and Haris Zafeiropoulos. 2021.
\newblock \href {https://doi.org/10.4230/LIPIcs.SoCG.2021.21} {Geometric algorithms for sampling the flux space of metabolic networks}.
\newblock In \emph{International Symposium on Computational Geometry (SoCG 2021)}, pages 21:1--21:16.

\bibitem[{Cousins and Vempala(2016)}]{cousins2016practical}
Ben Cousins and Santosh Vempala. 2016.
\newblock \href {https://doi.org/10.1007/s12532-015-0097-z} {A practical volume algorithm}.
\newblock \emph{Mathematical Programming Computation}, 8(2):133--160.

\bibitem[{Cousins and Vempala(2014)}]{cousins2014gaussian-cooling}
Benjamin~R. Cousins and Santosh~S. Vempala. 2014.
\newblock \href {https://api.semanticscholar.org/CorpusID:9673496} {Gaussian cooling and o*(n3) algorithms for volume and gaussian volume}.
\newblock \emph{SIAM J. Comput.}, 47:1237--1273.

\bibitem[{Dathathri et~al.(2020)Dathathri, Madotto, Lan, Hung, Frank, Molino, Yosinski, and Liu}]{dathathri2020plug}
Sumanth Dathathri, Andrea Madotto, Janice Lan, Jane Hung, Eric Frank, Piero Molino, Jason Yosinski, and Rosanne Liu. 2020.
\newblock \href {https://openreview.net/forum?id=H1edEyBKDS} {Plug and play language models: A simple approach to controlled text generation}.
\newblock In \emph{International Conference on Learning Representations}.

\bibitem[{Deng et~al.(2020)Deng, Bakhtin, Ott, Szlam, and Ranzato}]{Deng2020Residual}
Yuntian Deng, Anton Bakhtin, Myle Ott, Arthur Szlam, and Marc'Aurelio Ranzato. 2020.
\newblock \href {https://openreview.net/forum?id=B1l4SgHKDH} {Residual energy-based models for text generation}.
\newblock In \emph{International Conference on Learning Representations}.

\bibitem[{Diaconis et~al.(2000)Diaconis, Holmes, and Neal}]{diaconis-2000}
Persi Diaconis, Susan Holmes, and Radford~M. Neal. 2000.
\newblock \href {https://doi.org/10.1214/aoap/1019487508} {{Analysis of a nonreversible Markov chain sampler}}.
\newblock \emph{The Annals of Applied Probability}, 10(3):726 -- 752.

\bibitem[{Du et~al.(2023)Du, Torroba~Hennigen, Pimentel, Meister, Eisner, and Cotterell}]{du2023measure}
Li~Du, Lucas Torroba~Hennigen, Tiago Pimentel, Clara Meister, Jason Eisner, and Ryan Cotterell. 2023.
\newblock \href {https://doi.org/10.18653/v1/2023.acl-long.543} {A measure-theoretic characterization of tight language models}.
\newblock In \emph{Proceedings of the 61st Annual Meeting of the Association for Computational Linguistics (Volume 1: Long Papers)}, pages 9744--9770, Toronto, Canada. Association for Computational Linguistics.

\bibitem[{Du and Mordatch(2019)}]{du2019implicit}
Yilun Du and Igor Mordatch. 2019.
\newblock \href {https://proceedings.neurips.cc/paper_files/paper/2019/file/378a063b8fdb1db941e34f4bde584c7d-Paper.pdf} {Implicit generation and modeling with energy based models}.
\newblock In \emph{Advances in Neural Information Processing Systems}, volume~32.

\bibitem[{Duane et~al.(1987)Duane, Kennedy, Pendleton, and Roweth}]{duane1987}
Simon Duane, A.~D. Kennedy, Brian~J. Pendleton, and Duncan Roweth. 1987.
\newblock \href {https://doi.org/https://doi.org/10.1016/0370-2693(87)91197-X} {Hybrid {M}onte {C}arlo}.
\newblock \emph{Physics Letters B}, 195(2):216--222.

\bibitem[{Durrett(2019)}]{durrett_2019}
Rick Durrett. 2019.
\newblock \href {https://doi.org/10.1017/9781108591034} {\emph{Probability: Theory and Examples}}, 5$^{\text{th}}$ edition.
\newblock Cambridge Series in Statistical and Probabilistic Mathematics. Cambridge University Press.

\bibitem[{Dyer and Frieze(1988)}]{dyer1988polytope}
Martin~E. Dyer and Alan~M. Frieze. 1988.
\newblock \href {https://api.semanticscholar.org/CorpusID:35165656} {On the complexity of computing the volume of a polyhedron}.
\newblock \emph{SIAM J. Comput.}, 17:967--974.

\bibitem[{Emiris and Fisikopoulos(2013)}]{emiris2013polytope-volume}
Ioannis~Z. Emiris and Vissarion Fisikopoulos. 2013.
\newblock \href {https://api.semanticscholar.org/CorpusID:372936} {Efficient random-walk methods for approximating polytope volume}.
\newblock \emph{Proceedings of the thirtieth annual symposium on Computational geometry}.

\bibitem[{Emiris and Fisikopoulos(2018)}]{emiris2018practical-polytope}
Ioannis~Z. Emiris and Vissarion Fisikopoulos. 2018.
\newblock \href {https://api.semanticscholar.org/CorpusID:49291768} {Practical polytope volume approximation}.
\newblock \emph{ACM Transactions on Mathematical Software (TOMS)}, 44:1 -- 21.

\bibitem[{Geman and Geman(1984)}]{geman1984gibbs}
Stuart Geman and Donald Geman. 1984.
\newblock \href {https://doi.org/10.1109/TPAMI.1984.4767596} {Stochastic relaxation, gibbs distributions, and the bayesian restoration of images}.
\newblock \emph{IEEE Transactions on Pattern Analysis and Machine Intelligence}, PAMI-6(6):721--741.

\bibitem[{Ghazvininejad et~al.(2017)Ghazvininejad, Shi, Priyadarshi, and Knight}]{ghazvininejad-etal-2017-hafez}
Marjan Ghazvininejad, Xing Shi, Jay Priyadarshi, and Kevin Knight. 2017.
\newblock \href {https://aclanthology.org/P17-4008} {{H}afez: an interactive poetry generation system}.
\newblock In \emph{Proceedings of {ACL} 2017, System Demonstrations}, pages 43--48, Vancouver, Canada. Association for Computational Linguistics.

\bibitem[{Goyal et~al.(2022)Goyal, Dyer, and Berg-Kirkpatrick}]{goyal2022exposing}
Kartik Goyal, Chris Dyer, and Taylor Berg-Kirkpatrick. 2022.
\newblock \href {https://openreview.net/forum?id=6PvWo1kEvlT} {Exposing the implicit energy networks behind masked language models via metropolis--hastings}.
\newblock In \emph{International Conference on Learning Representations}.

\bibitem[{Grathwohl et~al.(2021)Grathwohl, Swersky, Hashemi, Duvenaud, and Maddison}]{grathwohl2021oops}
Will Grathwohl, Kevin Swersky, Milad Hashemi, David Duvenaud, and Chris Maddison. 2021.
\newblock \href {https://proceedings.mlr.press/v139/grathwohl21a.html} {Oops i took a gradient: Scalable sampling for discrete distributions}.
\newblock In \emph{Proceedings of the 38th International Conference on Machine Learning}, volume 139 of \emph{Proceedings of Machine Learning Research}, pages 3831--3841. PMLR.

\bibitem[{Grenander and Miller(1994)}]{grenander1994langevin}
Ulf Grenander and Michael~I. Miller. 1994.
\newblock \href {http://www.jstor.org/stable/2346184} {Representations of knowledge in complex systems}.
\newblock \emph{Journal of the Royal Statistical Society. Series B (Methodological)}, 56(4):549--603.

\bibitem[{Hastings(1970)}]{hastings1970}
W.~K. Hastings. 1970.
\newblock \href {https://doi.org/10.1093/biomet/57.1.97} {{Monte Carlo sampling methods using Markov chains and their applications}}.
\newblock \emph{Biometrika}, 57(1):97--109.

\bibitem[{He et~al.(2016)He, De~Sa, Mitliagkas, and R\'{e}}]{he2016scan-order}
Bryan He, Christopher De~Sa, Ioannis Mitliagkas, and Christopher R\'{e}. 2016.
\newblock \href {https://proceedings.neurips.cc/paper_files/paper/2016/file/e4da3b7fbbce2345d7772b0674a318d5-Paper.pdf} {Scan order in {G}ibbs sampling: Models in which it matters and bounds on how much}.
\newblock In \emph{Advances in Neural Information Processing Systems}, volume~29.

\bibitem[{Hinton(2002)}]{hinton2002}
Geoffrey~E. Hinton. 2002.
\newblock \href {https://doi.org/10.1162/089976602760128018} {Training products of experts by minimizing contrastive divergence}.
\newblock \emph{Neural Comput.}, 14(8):1771–1800.

\bibitem[{Hoffman and Gelman(2014)}]{hoffman2014nuts}
Matthew~D. Hoffman and Andrew Gelman. 2014.
\newblock \href {http://jmlr.org/papers/v15/hoffman14a.html} {The {No-U-Turn} sampler: Adaptively setting path lengths in {H}amiltonian {M}onte {C}arlo}.
\newblock \emph{Journal of Machine Learning Research}, 15(47):1593--1623.

\bibitem[{Holtzman et~al.(2018)Holtzman, Buys, Forbes, Bosselut, Golub, and Choi}]{holtzman-etal-2018-learning}
Ari Holtzman, Jan Buys, Maxwell Forbes, Antoine Bosselut, David Golub, and Yejin Choi. 2018.
\newblock \href {https://doi.org/10.18653/v1/P18-1152} {Learning to write with cooperative discriminators}.
\newblock In \emph{Proceedings of the 56th Annual Meeting of the Association for Computational Linguistics (Volume 1: Long Papers)}, pages 1638--1649, Melbourne, Australia. Association for Computational Linguistics.

\bibitem[{Horn and Johnson(2012)}]{horn2013}
Roger~A. Horn and Charles~R. Johnson. 2012.
\newblock \href {https://doi.org/10.1017/CBO9781139020411} {\emph{Matrix Analysis}}, 2$^{\text{nd}}$ edition.
\newblock Cambridge University Press.

\bibitem[{Horowitz(1991)}]{horowitz1991}
Alan~M. Horowitz. 1991.
\newblock \href {https://api.semanticscholar.org/CorpusID:121631021} {A generalized guided monte carlo algorithm}.
\newblock \emph{Physics Letters B}, 268:247--252.

\bibitem[{Kennedy(1990)}]{kennedy1990}
A.~D. Kennedy. 1990.
\newblock \href {https://doi.org/10.1007/978-1-4615-3784-7_14} {\emph{The Theory of Hybrid Stochastic Algorithms}}, pages 209--223. Springer US, Boston, MA.

\bibitem[{Kennedy and Pendleton(1991)}]{kennedy1991}
A.D. Kennedy and Brian Pendleton. 1991.
\newblock \href {https://doi.org/https://doi.org/10.1016/0920-5632(91)90893-J} {Acceptances and autocorrelations in hybrid monte carlo}.
\newblock \emph{Nuclear Physics B - Proceedings Supplements}, 20:118--121.

\bibitem[{Keskar et~al.(2019)Keskar, McCann, Varshney, Xiong, and Socher}]{keskar2019ctrl}
Nitish~Shirish Keskar, Bryan McCann, Lav~R. Varshney, Caiming Xiong, and Richard Socher. 2019.
\newblock \href {http://arxiv.org/abs/1909.05858} {Ctrl: A conditional transformer language model for controllable generation}.

\bibitem[{Krause et~al.(2021)Krause, Gotmare, McCann, Keskar, Joty, Socher, and Rajani}]{krause-etal-2021-gedi-generative}
Ben Krause, Akhilesh~Deepak Gotmare, Bryan McCann, Nitish~Shirish Keskar, Shafiq Joty, Richard Socher, and Nazneen~Fatema Rajani. 2021.
\newblock \href {https://doi.org/10.18653/v1/2021.findings-emnlp.424} {{G}e{D}i: Generative discriminator guided sequence generation}.
\newblock In \emph{Findings of the Association for Computational Linguistics: EMNLP 2021}, pages 4929--4952, Punta Cana, Dominican Republic. Association for Computational Linguistics.

\bibitem[{Kumar et~al.(2022)Kumar, Paria, and Tsvetkov}]{kumar-2022}
Sachin Kumar, Biswajit Paria, and Yulia Tsvetkov. 2022.
\newblock \href {https://doi.org/10.48550/ARXIV.2205.12558} {Constrained sampling from language models via langevin dynamics in embedding spaces}.
\newblock In \emph{Proceedings of the 2022 Conference on Empirical Methods in Natural Language Processing}. Association for Computational Linguistics.

\bibitem[{LeCun et~al.(2007)LeCun, Chopra, Hadsell, Ranzato, and Huang}]{lecun2006ebm}
Yann LeCun, Sumit Chopra, Raia Hadsell, Marc'Aurelio Ranzato, and Fu~Jie Huang. 2007.
\newblock \href {https://doi.org/10.7551/mitpress/7443.003.0014} {{Energy-Based Models}}.
\newblock In \emph{{Predicting Structured Data}}. The MIT Press.

\bibitem[{Levin and Peres(2017)}]{levin2017}
David~A. Levin and Yuval Peres. 2017.
\newblock \href {http://pages.uoregon.edu/dlevin/MARKOV/} {\emph{Markov Chains and Mixing Times}}, 2\textsuperscript{nd} edition.
\newblock American Mathematical Soc.

\bibitem[{Lin et~al.(2021)Lin, Jaech, Li, Gormley, and Eisner}]{lin2021limitations}
Chu-Cheng Lin, Aaron Jaech, Xin Li, Matthew~R. Gormley, and Jason Eisner. 2021.
\newblock \href {https://doi.org/10.18653/v1/2021.naacl-main.405} {Limitations of autoregressive models and their alternatives}.
\newblock In \emph{Proceedings of the 2021 Conference of the North American Chapter of the Association for Computational Linguistics: Human Language Technologies}, pages 5147--5173, Online. Association for Computational Linguistics.

\bibitem[{MacKay(2003)}]{mackay2003}
David J.~C. MacKay. 2003.
\newblock \href {http://www.cambridge.org/0521642981} {\emph{Information Theory, Inference, and Learning Algorithms}}.
\newblock Cambridge University Press.
\newblock Available from {\url{http://www.inference.phy.cam.ac.uk/mackay/itila/}}.

\bibitem[{Martens and Sutskever(2010)}]{martens-sutskever2010gits}
James Martens and Ilya Sutskever. 2010.
\newblock \href {https://proceedings.mlr.press/v9/martens10a.html} {Parallelizable sampling of markov random fields}.
\newblock In \emph{Proceedings of the Thirteenth International Conference on Artificial Intelligence and Statistics}, volume~9 of \emph{Proceedings of Machine Learning Research}, pages 517--524, Chia Laguna Resort, Sardinia, Italy. PMLR.

\bibitem[{Metropolis et~al.(1953)Metropolis, Rosenbluth, Rosenbluth, Teller, and Teller}]{metropolis1953equation}
Nicholas Metropolis, Arianna~W. Rosenbluth, Marshall~N. Rosenbluth, Augusta~H. Teller, and Edward Teller. 1953.
\newblock \href {https://doi.org/10.1063/1.1699114} {Equation of state calculations by fast computing machines}.
\newblock \emph{The Journal of Chemical Physics}, 21(6):1087--1092.

\bibitem[{Mimno and Thompson(2017)}]{mimno-thompson-2017-strange}
David Mimno and Laure Thompson. 2017.
\newblock \href {https://doi.org/10.18653/v1/D17-1308} {The strange geometry of skip-gram with negative sampling}.
\newblock In \emph{Proceedings of the 2017 Conference on Empirical Methods in Natural Language Processing}, pages 2873--2878, Copenhagen, Denmark. Association for Computational Linguistics.

\bibitem[{Mireshghallah et~al.(2022)Mireshghallah, Goyal, and Berg-Kirkpatrick}]{mireshghallah-etal-2022-mix}
Fatemehsadat Mireshghallah, Kartik Goyal, and Taylor Berg-Kirkpatrick. 2022.
\newblock \href {https://doi.org/10.18653/v1/2022.acl-long.31} {Mix and match: Learning-free controllable text generation using energy language models}.
\newblock In \emph{Proceedings of the 60th Annual Meeting of the Association for Computational Linguistics (Volume 1: Long Papers)}, pages 401--415, Dublin, Ireland. Association for Computational Linguistics.

\bibitem[{{Mohasel Afshar} and Domke(2015)}]{rhmc}
Hadi {Mohasel Afshar} and Justin Domke. 2015.
\newblock \href {https://proceedings.neurips.cc/paper/2015/file/8303a79b1e19a194f1875981be5bdb6f-Paper.pdf} {Reflection, refraction, and {H}amiltonian {M}onte {C}arlo}.
\newblock In \emph{Advances in Neural Information Processing Systems}, volume~28.

\bibitem[{Neal(1993)}]{neal1993probabilistic}
Radford~M. Neal. 1993.
\newblock \href {https://bayes.wustl.edu/Manual/RadfordNeal.review.pdf} {\emph{Probabilistic Inference using {M}arkov chain {M}onte {C}arlo methods}}.
\newblock Department of Computer Science, University of Toronto Toronto, ON, Canada.

\bibitem[{Neal(2011)}]{neal2011hmc}
Radford~M. Neal. 2011.
\newblock \href {https://doi.org/10.1201/b10905} {{MCMC} using {H}amiltonian dynamics}.
\newblock In Steve Brooks, Andrew Gelman, Galin Jones, and Xiao-Li Meng, editors, \emph{Handbook of Markov Chain Monte Carlo}, chapter~5. Chapman and Hall/{CRC}.

\bibitem[{Newman and Barkema(1999)}]{newman1999monte}
M.E.J. Newman and G.T. Barkema. 1999.
\newblock \href {https://books.google.ch/books?id=KKL2nQEACAAJ} {\emph{Monte Carlo Methods in Statistical Physics}}.
\newblock Oxford: Clarendon Press.

\bibitem[{Nishimura et~al.(2020)Nishimura, Dunson, and Lu}]{nishimura2020dhmc}
Akihiko Nishimura, David~B. Dunson, and Jianfeng Lu. 2020.
\newblock \href {https://doi.org/10.1093/biomet/asz083} {{Discontinuous Hamiltonian Monte Carlo for discrete parameters and discontinuous likelihoods}}.
\newblock \emph{Biometrika}, 107(2):365--380.

\bibitem[{Novikova et~al.(2017)Novikova, Du{\v{s}}ek, and Rieser}]{novikova-etal-2017-e2e}
Jekaterina Novikova, Ond{\v{r}}ej Du{\v{s}}ek, and Verena Rieser. 2017.
\newblock \href {https://doi.org/10.18653/v1/W17-5525} {The {E}2{E} dataset: New challenges for end-to-end generation}.
\newblock In \emph{Proceedings of the 18th Annual {SIG}dial Meeting on Discourse and Dialogue}, pages 201--206, Saarbr{\"u}cken, Germany. Association for Computational Linguistics.

\bibitem[{Pakman and Paninski(2013)}]{pakman-paninski-nips2013}
Ari Pakman and Liam Paninski. 2013.
\newblock \href {https://proceedings.neurips.cc/paper_files/paper/2013/file/a7d8ae4569120b5bec12e7b6e9648b86-Paper.pdf} {Auxiliary-variable exact hamiltonian monte carlo samplers for binary distributions}.
\newblock In \emph{Advances in Neural Information Processing Systems}, volume~26. Curran Associates, Inc.

\bibitem[{Pakman and Paninski(2014)}]{pakman2012gaussian-ehmc}
Ari Pakman and Liam Paninski. 2014.
\newblock \href {https://arxiv.org/abs/1208.4118} {Exact hamiltonian monte carlo for truncated multivariate gaussians}.
\newblock \emph{Journal of Computational and Graphical Statistics}, 23(2):518--542.

\bibitem[{Phan et~al.(2019)Phan, Pradhan, and Jankowiak}]{phan2019numpyro}
Du~Phan, Neeraj Pradhan, and Martin Jankowiak. 2019.
\newblock Composable effects for flexible and accelerated probabilistic programming in numpyro.
\newblock \emph{arXiv preprint arXiv:1912.11554}.

\bibitem[{Qin et~al.(2022)Qin, Welleck, Khashabi, and Choi}]{qin2022cold}
Lianhui Qin, Sean Welleck, Daniel Khashabi, and Yejin Choi. 2022.
\newblock \href {https://doi.org/10.48550/ARXIV.2202.11705} {{COLD} decoding: Energy-based constrained text generation with {L}angevin dynamics}.
\newblock In \emph{Advances in Neural Information Processing Systems}.

\bibitem[{Radford et~al.(2019)Radford, Wu, Child, Luan, Amodei, and Sutskever}]{radford2019language}
Alec Radford, Jeff Wu, Rewon Child, David Luan, Dario Amodei, and Ilya Sutskever. 2019.
\newblock Language models are unsupervised multitask learners.

\bibitem[{Raffel et~al.(2020)Raffel, Shazeer, Roberts, Lee, Narang, Matena, Zhou, Li, and Liu}]{raffel2020t5}
Colin Raffel, Noam Shazeer, Adam Roberts, Katherine Lee, Sharan Narang, Michael Matena, Yanqi Zhou, Wei Li, and Peter~J. Liu. 2020.
\newblock \href {http://jmlr.org/papers/v21/20-074.html} {Exploring the limits of transfer learning with a unified text-to-text transformer}.
\newblock \emph{Journal of Machine Learning Research}, 21(140):1--67.

\bibitem[{Rhodes and Gutmann(2022)}]{rhodes2022enhanced}
Benjamin Rhodes and Michael~U. Gutmann. 2022.
\newblock \href {https://openreview.net/forum?id=j2Mid5hFUJ} {Enhanced gradient-based {MCMC} in discrete spaces}.
\newblock \emph{Transactions on Machine Learning Research}.

\bibitem[{Roberts and Rosenthal(1998)}]{roberts1998optimal}
Gareth~O. Roberts and Jeffrey~S. Rosenthal. 1998.
\newblock \href {https://api.semanticscholar.org/CorpusID:5831882} {Optimal scaling of discrete approximations to langevin diffusions}.
\newblock \emph{Journal of the Royal Statistical Society: Series B (Statistical Methodology)}, 60.

\bibitem[{Roberts and Tweedie(1996)}]{roberts1996exponential}
Gareth~O. Roberts and Richard~L. Tweedie. 1996.
\newblock \href {https://projecteuclid.org/journals/bernoulli/volume-2/issue-4/Exponential-convergence-of-Langevin-distributions-and-their-discrete-approximations/bj/1178291835.full} {Exponential convergence of {L}angevin distributions and their discrete approximations}.
\newblock \emph{Bernoulli}, 2(4):341 -- 363.

\bibitem[{Sohl-Dickstein et~al.(2014)Sohl-Dickstein, Mudigonda, and DeWeese}]{sohl-dickstein-2014}
Jascha Sohl-Dickstein, Mayur Mudigonda, and Michael DeWeese. 2014.
\newblock \href {https://proceedings.mlr.press/v32/sohl-dickstein14.html} {Hamiltonian monte carlo without detailed balance}.
\newblock In \emph{Proceedings of the 31st International Conference on Machine Learning}, volume~32 of \emph{Proceedings of Machine Learning Research}, pages 719--726, Bejing, China. PMLR.

\bibitem[{Song and Ermon(2020)}]{song2020score-based}
Yang Song and Stefano Ermon. 2020.
\newblock \href {https://proceedings.neurips.cc/paper/2020/file/92c3b916311a5517d9290576e3ea37ad-Paper.pdf} {Improved techniques for training score-based generative models}.
\newblock In \emph{Advances in Neural Information Processing Systems}, volume~33, pages 12438--12448. Curran Associates, Inc.

\bibitem[{Welling and Teh(2011)}]{welling2011langevin}
Max Welling and Yee~Whye Teh. 2011.
\newblock \href {https://www.stats.ox.ac.uk/~teh/research/compstats/WelTeh2011a.pdf} {Bayesian learning via stochastic gradient {L}angevin dynamics}.
\newblock In \emph{Proceedings of the 28th International Conference on International Conference on Machine Learning}, page 681–688, Madison, WI, USA. Omnipress.

\bibitem[{Yang and Klein(2021)}]{yang-klein-2021-fudge}
Kevin Yang and Dan Klein. 2021.
\newblock \href {https://doi.org/10.18653/v1/2021.naacl-main.276} {{FUDGE}: Controlled text generation with future discriminators}.
\newblock In \emph{Proceedings of the 2021 Conference of the North American Chapter of the Association for Computational Linguistics: Human Language Technologies}, pages 3511--3535, Online. Association for Computational Linguistics.

\bibitem[{Zanella(2020)}]{zanella2020}
Giacomo Zanella. 2020.
\newblock \href {https://doi.org/10.1080/01621459.2019.1585255} {Informed proposals for local {MCMC} in discrete spaces}.
\newblock \emph{Journal of the American Statistical Association}, 115(530):852--865.

\bibitem[{Zhang et~al.(2022)Zhang, Liu, and Liu}]{zhang2022}
Ruqi Zhang, Xingchao Liu, and Qiang Liu. 2022.
\newblock \href {https://proceedings.mlr.press/v162/zhang22t.html} {A {L}angevin-like sampler for discrete distributions}.
\newblock In \emph{Proceedings of the 39th International Conference on Machine Learning}, volume 162 of \emph{Proceedings of Machine Learning Research}, pages 26375--26396. PMLR.

\bibitem[{Zhang et~al.(2012)Zhang, Ghahramani, Storkey, and Sutton}]{zhang2012gits}
Yichuan Zhang, Zoubin Ghahramani, Amos~J Storkey, and Charles Sutton. 2012.
\newblock \href {https://proceedings.neurips.cc/paper_files/paper/2012/file/c913303f392ffc643f7240b180602652-Paper.pdf} {Continuous relaxations for discrete hamiltonian monte carlo}.
\newblock In \emph{Advances in Neural Information Processing Systems}, volume~25. Curran Associates, Inc.

\end{thebibliography}
\bibliographystyle{acl_natbib}

\newpage
\appendix
\onecolumn

\section{Related Works}
\label{sec:related-works}

\paragraph{Controlled Generation.} Since the introduction of large pretrained language models, controlled generation, the ability to enforce controls during the text generation process has become an important research direction \citep[][\textit{inter alia}]{keskar2019ctrl,dathathri2020plug,krause-etal-2021-gedi-generative}. 
Earlier approaches in this direction includes weighted decoding \citep{ghazvininejad-etal-2017-hafez,holtzman-etal-2018-learning,yang-klein-2021-fudge}, which adjusts the language model score of each sequence with a function that measures how well it adheres to its control objectives and then try to decode the high scoring sequences.
More recently, several works formulated energy-based models using pretrained language models \citep{Deng2020Residual,goyal2022exposing} to express the control objective \citep{kumar-2022,qin2022cold,amini2023structured,mireshghallah-etal-2022-mix} and attempted to apply MCMC algorithms to sample from such sequence distribution.
When the underlying pretrained language model is a masked language model \citep{mireshghallah-etal-2022-mix}, the masked distributions are highly effective as approximations to the true conditionals and hence the Metropolis--Hastings corrected Gibbs-like scheme may work well without the need of gradient \citep{goyal2022exposing}.
However, when the underlying is causal \citep{kumar-2022,qin2022cold,amini2023structured}, which is the subject of this paper, there is no obvious choice of proposal distributions as discussed in \cref{sec:sampling-from-ebms}, and hence gradient information becomes valuable for deriving a proposal distribution without additional training.

\paragraph{Gradient-based Sampling} Our work is also related to the line of research that makes use of gradient information to sample from complex distributions \citep{duane1987,neal1993probabilistic,grenander1994langevin}.
In Bayesian inference, gradient-based samplers \cite{neal2011hmc,hoffman2014nuts} are known to be highly effective when sampling from high-dimensional continuous distributions \cite{carpenter2017stan,bingham2018pyro,phan2019numpyro}.
But it has been shown to be a difficult problem to adapt these algorithms in the discrete setting \citep{roberts1996exponential,roberts1998optimal}, with previous approaches including continuous relaxation within the discrete spaces \citep{pakman-paninski-nips2013} using discontinuous Hamiltonian Monte Carlo \citep{pakman2012gaussian-ehmc,rhmc,nishimura2020dhmc}, continuous relaxation via the ``Gaussian Integral Trick'' \citep{martens-sutskever2010gits,zhang2012gits}.
Specifically, the \pncg proposed in our work is a generalization of NCG proposed in \citet{rhodes2022enhanced} and the D-MALA proposed in \citet{zhang2022}, with the difference being using the $p$-norm constraint instead of the standard $\ell_2$ norm.
Our Gibbs-with-Langevin algorithm is also loosely related to the Gibbs-with-Gradient method proposed in \citet{grathwohl2021oops}, which we found to have near zero acceptance rate when applied to our setting.
We note that a range of recently proposed gradient-based samplers \citep{grathwohl2021oops,zhang2022,rhodes2022enhanced} are in some way connected to the locally balanced proposal from \citep{zanella2020}.

\section{On High-Dimensional Integration in Embedding Spaces}
\label{sec:high-dim-int}

\subsection{The Problem of Continuous Relaxation and High-Dimensional Integration}
\label{sec:high-dim-int-in-mcmc}

A common strategy to continuous relax discrete spaces is to map the discrete points into a continuous space and apply continuous gradient-based sampling algorithms \citep{pakman-paninski-nips2013,amini2023structured}.
This strategy gives rise to the problem of converting samples from continuous algorithms into discrete ones.
This problem is easier when the underlying discrete space is regularly shaped as in Ising model \citep{pakman-paninski-nips2013} where the projection function is as simple as the sign function $\sgn(\cdot)$.
When the underlying discrete space is irregularly shaped such as the word embedding space, one can use the Euclidean projection to convert a continuous sample $\vy\in\R^d$ into a discrete one $\vx\in\calX$, as in
\begin{align}
    \vx=\proj_{\calX} \vy.
\end{align}
This projection is used in both \citet{amini2023structured} and \citet{kumar-2022} and it creates a number of problems.

\paragraph{SVS.} In the case of \svs, \citet{amini2023structured} realized that the projection created a piecewise continuous relaxation, with each continuous region corresponding to a Voronoi cell
\begin{align}
    V_i = \left\{
        \vy~:~\|\vy-\vx_i\|_2\leq \|\vy-\vx_{i'}\|_2, \forall~i'\not=i
    \right\}
\end{align}
centering at a word embedding $\vx_i$. 
\citet{amini2023structured} then uses Gaussian augmentation within the Voronoi cells to apply gradient-based samplers.
To ensure that the continuously relaxed measure matches original the discrete measure, the underlying measure needs to be adjusted by the integral of the Gaussian truncated by the high-dimensional Voronoi polytope, otherwise known as the Gaussian volume of a polytope, defined as
\begin{align}
    \int_{V_i} \gamma^d(\vy; \vx, \sigma^2)\mathrm{d}\vy.
\end{align}
where $\gamma^d(\,\cdot\,; \vx, \sigma^2)$ denotes the $d$-dimensional Gaussian density centered at $\vx$ with variance $\sigma^2$.

\paragraph{\mucola.} By using the Euclidean projection operator in its update equation:
\mucolaUpdate*

\mucola similarly identifies each Voronoi region in $\R^d$ with the word embedding at its center. As we demonstrated in \cref{sec:faithful-text-samplers}, \mucola doesn't sample from its intended language distribution.
An obvious idea is then to apply Metropolis-Hasting correction to \mucola, which requires one to compute $q_\mucola(\vx_j\mid\vx_i)$ in the Metropolis-Hasting acceptance probability \cref{eq:metropolis}. Observing that
\begin{align}
    \proj_\calX \left(
        \vx_i-\frac{\alpha}{2}\nabla U(\vx_i) + \sqrt{\alpha} \vxi
    \right) = \vx_j ~\Leftrightarrow~ \vx_i-\frac{\alpha}{2}\nabla U(\vx_i) + \sqrt{\alpha} \vxi \in V_j,
\end{align}
we realize that computing $q_\mucola(\vx_j\mid\vx_i)$ is equivalent to computing the following integral
\begin{align}
\int_{V_j} \gamma^d\left(\vy; \vx_i-\frac{\alpha}{2}\nabla U(\vx_i), 1 \right) \mathrm{d}\vy
\end{align}
which is again the same high dimensional integral we encountered in \svs.

\subsection{The Difficulty of High-Dimensional Integration}
\label{sec:high-dim-infeasible}

In general, computing the volume of an explicit polytope is \#P-hard \citep{dyer1988polytope}, which makes exact computation infeasible for dimensions as high as that of GPT-2 or BERT.
Recent research on approximated high-dimensional integration shows great promises \citep{cousins2014gaussian-cooling,emiris2013polytope-volume} and saw such algorithms \citep{cousins2016practical,emiris2018practical-polytope} improved to the extent that they can be employed in various applied sciences \citep{chalkis2021metabolic}. Unfortunately, in our experimentation with these algorithms, we found that they can barely scales to dimensions beyond 100, not to mention the dimensions in GPT-2 or BERT, which are at the scale of $10^3$.
We therefore conclude that, at the current moment, the state of research in high-dimensional integration doesn't yet allow us to feasibly compute the relevant quantities so that \svs and \mucola can sample from the correct distribution.

\section{Derivation of $p$-NCG}
\label{sec:pncg-derivation}

We start with \cref{eq:pncg-form-1}
\begin{subequations}
\begin{align}
    q(\vx'\mid\vx) 
    &=\exp\left(
        -\frac{1}{2\alpha}
        \left\|
            \vx'-\left(\vx-\frac{\alpha}{2}\nabla U(\vx)\right)
        \right\|_2^2
    \right) \\
    &=\exp\left(
        -\frac{1}{2\alpha}
        \left\|
            (\vx'-\vx)+\frac{\alpha}{2}\nabla U(\vx)
        \right\|_2^2
    \right) \label{eq:pncg-derive-1}
\end{align}
\end{subequations}
where 
\begin{subequations}
\begin{align}
    &\phantom{=}\frac{1}{2\alpha}
        \left\|
            (\vx'-\vx)+\frac{\alpha}{2}\nabla U(\vx)
        \right\|_2^2 \\
    &=\frac{1}{2\alpha} \|\vx'-\vx\|_2^2 
        + 2\cdot\frac{1}{2\cancel{\alpha}}\left\langle \vx'-\vx,\frac{\cancel{\alpha}}{2}\nabla U(\vx) \right\rangle
        + \frac{1}{2\alpha}\cdot \frac{\alpha^2}{4} \left\| \nabla U(\vx)\right\|_2^2 \\
    &=\nabla U(\vx)^\top (\vx'-\vx) + \frac{1}{2\alpha} \|\vx'-\vx\|_2^2 + {\color{gray} \frac{\alpha}{8} \left\| \nabla U(\vx)\right\|_2^2 } \label{eq:pncg-derive-2}
\end{align}
\end{subequations}
Substituting \cref{eq:pncg-derive-2} into \cref{eq:pncg-derive-1}, we get
\begin{align}
    q(\vx'\mid\vx) \propto \exp\left(
        -\nabla U(\vx)^\top (\vx'-\vx) - \frac{1}{2\alpha} \|\vx'-\vx\|_2^2 - {\color{gray} \frac{\alpha}{8} \left\| \nabla U(\vx)\right\|_2^2 }
    \right)
\end{align}

Notice that the last term ${\color{gray} \frac{\alpha}{8} \left\| \nabla U(\vx)\right\|_2^2 }$ only contains $\vx$ and does not involve $\vx'$, so it will cancel with the same term in the normalizing constant. 
This means that we can omit this term from the proposal distribution. Taking this into account, we get the alternate form of the proposal as given in \cref{eq:pncg-form-2}:
\begin{align}
    q(\vx'\mid\vx) \propto \exp\left(
        -\nabla U(\vx)^\top (\vx'-\vx) - \frac{1}{2\alpha} \|\vx'-\vx\|_2^2
    \right).
\end{align}

\section{Proof of \cref{thm:pncg-convergence}}
\label{sec:pncg-convergence-proof}

\pncgConvergence*

We adapt the proof strategy from the proof of Theorem 1 in \citet{zanella2020} and from \citet{zhang2022}.

\begin{proof}
To avoid confusion, we use $q_\alpha(\cdot\mid\vx)$ to denote the proposal in \cref{eq:pncg-final-form} with step size $\alpha$, i.e.,
\begin{align}
    \label{eq:pncg-alpha}
    q_\alpha(\vx'\mid\vx)\propto \exp\left(
        -\frac{1}{2}\nabla U(\vx)^\top(\vx'-\vx)
        -\frac{1}{2\alpha}\|\vx'-\vx\|_p^p
    \right).
\end{align}

We first note that, for $\alpha>0$, the proposal $q_\alpha$ is dense in the sense that $q_\alpha(\vx'\mid\vx)>0$ for all $\vx,\vx'\in\calX$. This implies that the chain is irreducible and aperiodic, which guarantees that there must be a unique stationary distribution.

Let $\pi(\vx)\propto \exp\left( \vx^\top A \vx + \vb^\top \vx \right)$ be a discrete log-quadratic distribution. In this case, the energy function is $U(\vx)=-\vx^\top A \vx - \vb^\top \vx$.
Since $U(\vx)$ is a quadratic function, the second-order Taylor expansion is exact, which means
\begin{align}
    \label{eq:pncg-conv-proof-taylor}
    U(\vx')=U(\vx)+\nabla U(\vx)^\top (\vx'-\vx) + \frac{1}{2} (\vx'-\vx)^\top \nabla^2 U(\vx) (\vx'-\vx).
\end{align}
Rearranging \cref{eq:pncg-conv-proof-taylor}, we get
\begin{align}
    \nabla U(\vx)^\top (\vx'-\vx) &= U(\vx')-U(\vx) - \frac{1}{2} (\vx'-\vx)^\top \nabla^2 U(\vx) (\vx'-\vx) \\
    \intertext{which is equivalent to}
    \frac{1}{2}\nabla U(\vx)^\top (\vx'-\vx) &= \frac{1}{2}\left(U(\vx')-U(\vx)\right) - \frac{1}{4} (\vx'-\vx)^\top \nabla^2 U(\vx) (\vx'-\vx) \\
    &= \frac{1}{2}\left(U(\vx')-U(\vx)\right) + \frac{1}{2} (\vx'-\vx)^\top A (\vx'-\vx). \label{eq:pncg-conv-proof-1} &\justification{$\nabla^2 U(\vx)=-2A$}
\end{align}
Using \cref{eq:pncg-conv-proof-1}, we can rewrite the proposal \cref{eq:pncg-alpha} as
\begin{restatable}{donothing}{alphaProposalNormalized}
\begin{equation}
    \label{eq:pncg-alpha-proposal-normalized}
    q_\alpha(\vx'\mid\vx)=\frac{1}{Z_\alpha(\vx)} \exp\left(
        -\frac{1}{2}\left(U(\vx')-U(\vx)\right) 
        - \frac{1}{2} (\vx'-\vx)^\top A (\vx'-\vx)
        -\frac{1}{2\alpha}\|\vx'-\vx\|_p^p
    \right)
\end{equation}
\end{restatable}
where
\begin{restatable}{donothing}{alphaProposalNormalizationConst}
\begin{align}
    \label{eq:pncg-alpha-proposal-normalization-const}
    Z_\alpha(\vx)=\sum_{\vy\in\calX} \exp\left(
        -\frac{1}{2}\left(U(\vy)-U(\vx)\right) 
        - \frac{1}{2} (\vy-\vx)^\top A (\vy-\vx)
        -\frac{1}{2\alpha}\|\vy-\vx\|_p^p
    \right).
\end{align}
\end{restatable}
Now, we suppose $\pi_\alpha$ is a reversing distribution with respect to $q_\alpha$ and try to solve for it. 
First, by the definition of reversibility,
\begin{align}
    \label{eq:alpha-detailed-balance}
    \pi_\alpha(\vx)q_\alpha(\vx'\mid\vx) &= \pi_\alpha(\vx')q_\alpha(\vx\mid\vx')
\end{align}
which, after substituting in \cref{eq:pncg-alpha-proposal-normalized}, expands to
\begin{align}
    & \frac{\pi_\alpha(\vx)}{Z_\alpha(\vx)} \exp\left(
        -\frac{1}{2}\left(U(\vx')-U(\vx)\right) 
        - \Ccancel[ETHBlue]{\frac{1}{2} (\vx'-\vx)^\top A (\vx'-\vx)}
        - \Ccancel[ETHRed]{\frac{1}{2\alpha}\|\vx'-\vx\|_p^p}
    \right) \nonumber \\
    = &\frac{\pi_\alpha(\vx')}{Z_\alpha(\vx')} \exp\left(
        -\frac{1}{2}\left(U(\vx)-U(\vx')\right) 
        - \Ccancel[ETHBlue]{\frac{1}{2} (\vx-\vx')^\top A (\vx-\vx')}
        - \Ccancel[ETHRed]{\frac{1}{2\alpha}\|\vx-\vx'\|_p^p}
    \right)
\end{align}
and simplifies to
\begin{align}
    && \frac{\pi_\alpha(\vx)}{Z_\alpha(\vx)} \exp\left(
        -\frac{1}{2}\left(U(\vx')-U(\vx)\right)
    \right) &= \frac{\pi_\alpha(\vx')}{Z_\alpha(\vx')} \exp\left(
        -\frac{1}{2}\left(U(\vx)-U(\vx')\right)
    \right) \label{eq:alpha-detailed-balance-intermediate} \\
    \Leftrightarrow && \frac{\pi_\alpha(\vx) }{Z_\alpha(\vx)} \exp(U(\vx))  &= \frac{\pi_\alpha(\vx') }{Z_\alpha(\vx')} \exp(U(\vx')) \\
    \Leftrightarrow && \frac{\pi_\alpha(\vx) }{Z_\alpha(\vx)} \cdot  {\frac{Z}{\exp(-U(\vx))}}
                    &= \frac{\pi_\alpha(\vx') }{Z_\alpha(\vx')} \cdot {\frac{Z}{\exp(-U(\vx')) }}
                    &\justification{$\textstyle Z\defeq\sum\nolimits_{\vx\in\calX} \exp(-U(\vx))$} \\
    \Leftrightarrow && \frac{\pi_\alpha(\vx) }{Z_\alpha(\vx) \pi(\vx)}
                    &= \frac{\pi_\alpha(\vx') }{Z_\alpha(\vx') \pi(\vx')}.
                    &\justification{$\pi(\vx)=\exp(-U(\vx))/Z$}
                    \label{eq:pncg-conv-proof-equal-chain}
\end{align}
\cref{eq:pncg-conv-proof-equal-chain} shows that $\frac{\pi_\alpha(\vx) }{Z_\alpha(\vx) \pi(\vx)}=c_\alpha$ for some constant $c_\alpha$ for all $\vx\in\calX$.
Noting that $\sum_{\vx\in\calX} \pi_\alpha(\vx)=1$, we can solve for $c_\alpha$ to be
\begin{align}
    1=\sum_{\vx\in\calX} \pi_\alpha(\vx) = \sum_{\vx\in\calX} c_\alpha Z_\alpha(\vx) \pi(\vx) = c_\alpha \sum_{\vx\in\calX} Z_\alpha(\vx) \pi(\vx)
\end{align}
which yields
\begin{align}
    c_\alpha = \frac{1}{\sum_{\vx\in\calX} Z_\alpha(\vx) \pi(\vx)}
\end{align}
and hence the reversing measure $\pi_\alpha$ should be
\begin{align}
    \pi_\alpha(\vx)=\frac{Z_\alpha(\vx) \pi(\vx)}{\sum_{\vy\in\calX} Z_\alpha(\vy) \pi(\vy)}.
    \label{eq:pncg-alpha-reversing-dist}
\end{align}
One can quickly verify that $\pi_\alpha$ as defined in \cref{eq:pncg-alpha-reversing-dist} indeed satisfies the detailed balance equation in \cref{eq:alpha-detailed-balance} and hence is indeed a reversing measure for $q_\alpha$. 
We can now conclude that $q_\alpha$ produces a reversible chain and that $\pi_\alpha$ is its unique stationary (and simultaneously reversing) measure.\footnote{One may notice at \cref{eq:alpha-detailed-balance-intermediate} that setting $\pi_\alpha(\vx)\propto \exp(-U(\vx))/Z_\alpha(\vx)$ will symmetrize both sides of the equation, resulting in detailed balance. This observation can avoid the last bit of calculation.}

Finally, to show the weak convergence, we observe that
\begin{align}
    \label{eq:pncg-proof-dirac-delta}
    \lim_{\alpha\to 0} \exp\left(
        -\frac{1}{2}\left(U(\vy)-U(\vx)\right) 
        - \frac{1}{2} (\vy-\vx)^\top A (\vy-\vx)
        -\frac{1}{2\alpha}\|\vy-\vx\|_p^p
    \right) = \begin{cases}
        0 & \vy\not=\vx \\
        1 & \vy=\vx
    \end{cases}=\delta_{\vx}(\vy)
\end{align}
where $\delta_{\vx}(\cdot)$ is the Dirac delta centered at $\vx$. This means that
\begin{align}
    &\lim_{\alpha\to 0} Z_\alpha(\vx) \\
    =&\lim_{\alpha\to 0} \sum_{\vy\in\calX} \exp\left(
        -\frac{1}{2}\left(U(\vy)-U(\vx)\right) 
        - \frac{1}{2} (\vy-\vx)^\top A (\vy-\vx)
        -\frac{1}{2\alpha}\|\vy-\vx\|_p^p
    \right) \\
    =& \sum_{\vy\in\calX} \delta_{\vx}(\vy)  &\justification{by \cref{eq:pncg-proof-dirac-delta}} \\
    =& 1.
\end{align}
Hence
\begin{align}
    \lim_{\alpha\to 0} \pi_\alpha(\vx)=\lim_{\alpha\to 0} \frac{Z_\alpha(\vx) \pi(\vx)}{\sum_{\vy\in\calX} Z_\alpha(\vy) \pi(\vy)}
    =\frac{\pi(\vx)}{\sum_{\vy\in\calX} \pi(\vy)} = \pi(\vx)
\end{align}
which shows that $\pi_\alpha$ converges to $\pi$ pointwise. It is a well-known result that, in the case of discrete distributions, pointwise convergence implies weak convergence.\footnote{See, for example, Exercise 3.2.11 in \citet{durrett_2019}.} Hence, $\pi_\alpha\to\pi$ weakly as $\alpha\to0$.
\end{proof}

\section{Proof of \cref{thm:pncg-mixing-time}}
\label{sec:pncg-mixing-time}

We state the \gershgorin disc theorem here for reference.
\begin{theorem}[\gershgorin disc theorem; Theorem 6.1.1 in \citealp{horn2013}]
\label{thm:gershgorin}
Given a matrix $P$ and denote its non-diagonal sum as $R_i=\sum_{j\not=i} |P_{ij}|$. Define the \gershgorin discs as
\begin{align}
    D(a_{ii},R_i)=\{z\in\C: |z-a_{ii}|\leq R_i\}.
\end{align}
Then, all eigenvalues of $P$ are in the union of the \gershgorin discs.
\end{theorem}

\pncgMixing*
\begin{proof}
Let $\pi(\vx)\propto \exp\left( \vx^\top A \vx + \vb^\top \vx \right)$ be a discrete log-quadratic distribution. Here, we let the energy function be $U(\vx)=-\vx^\top A \vx - \vb^\top \vx + \text{const}$. 
We additionally assume, without loss of generality, that $U(\vx)\leq0$ for all $\vx\in\calX$, since we can subtract a constant from the energy function of each state without altering the distribution.

We recall from the proof of \cref{thm:pncg-convergence} that the proposal can be rewritten as
\alphaProposalNormalized*
where
\alphaProposalNormalizationConst*

To apply the \gershgorin disc theorem, we first need to bound the non-diagonal mass in the transition matrix.
The non-diagonal mass, i.e., the probability of non-self-transition, is
\begin{align}
    &\sum_{\vy\not=\vx} q_\alpha(\vy\mid\vx) \\
    =& \frac{
        \sum_{\vy\not=\vx} \exp\left(
            -\frac{1}{2}\left(U(\vy)-U(\vx)\right) 
            - \frac{1}{2} (\vy-\vx)^\top A (\vy-\vx)
            -\frac{1}{2\alpha}\|\vy-\vx\|_p^p
        \right)
    }{
        \sum_{\vy\in\calX} \exp\left(
            -\frac{1}{2}\left(U(\vy)-U(\vx)\right) 
            - \frac{1}{2} (\vy-\vx)^\top A (\vy-\vx)
            -\frac{1}{2\alpha}\|\vy-\vx\|_p^p
        \right)
    } \\
    =& \frac{
        \sum_{\vy\not=\vx} \exp\left(
            -\frac{1}{2}\left(U(\vy)-U(\vx)\right) 
            - \frac{1}{2} (\vy-\vx)^\top A (\vy-\vx)
            -\frac{1}{2\alpha}\|\vy-\vx\|_p^p
        \right)
    }{
        1 + \sum_{\vy\not=\vx} \exp\left(
            -\frac{1}{2}\left(U(\vy)-U(\vx)\right) 
            - \frac{1}{2} (\vy-\vx)^\top A (\vy-\vx)
            -\frac{1}{2\alpha}\|\vy-\vx\|_p^p
        \right)
    } \\
    \leq& \sum_{\vy\not=\vx} \exp\left(
        -\frac{1}{2}\left(U(\vy)-U(\vx)\right) 
        - \frac{1}{2} (\vy-\vx)^\top A (\vy-\vx)
        -\frac{1}{2\alpha}\|\vy-\vx\|_p^p
    \right) \label{eq:non-self-transition-bound-1}
\end{align}
Without loss of generality, we can assume that  $A$ is symmetric because we can substitute $A$ with its symmetric part $\frac{1}{2}(A^\top+A)$ without changing any quantity of interest.
Then we can apply the Rayleigh-Ritz inequality,%
which states that, for any $\vv\not=0$,
\begin{align}
    \label{eq:rayleigh-ritz}
    \frac{\vv^\top A \vv}{\vv^\top\vv} \geq \lambda_{\min}(A).
\end{align}
We further define the useful quantity for $q\geq 1$,
\begin{align}\label{eq:dq-defn}
    d_q\defeq\inf_{\vx\not=\vx'\in\calX} \|\vx-\vx'\|_q^q.
\end{align}
Continuing from \cref{eq:non-self-transition-bound-1},
\begin{align}
    & \sum_{\vy\not=\vx} \exp\left(
        -\frac{1}{2}\left(U(\vy)-U(\vx)\right) 
        - \frac{1}{2} (\vy-\vx)^\top A (\vy-\vx)
        -\frac{1}{2\alpha}\|\vy-\vx\|_p^p
    \right) \\
    \leq& \sum_{\vy\not=\vx} \exp\left(
        -\frac{1}{2}\left(U(\vy)-U(\vx)\right) 
        - \frac{1}{2} \lambda_{\min}(A) \|\vy-\vx\|_2^2
        -\frac{1}{2\alpha}\|\vy-\vx\|_p^p
    \right)  &\justification{Rayleigh-Ritz, \cref{eq:rayleigh-ritz}} \\
    \leq& \sum_{\vy\not=\vx} \exp\left(
        -\frac{1}{2}\left(U(\vy)-U(\vx)\right) 
        - \frac{1}{2} \lambda_{\min}(A) d_2
        -\frac{1}{2\alpha} d_p
    \right) &\justification{definition of $d_q$, \cref{eq:dq-defn}} \\
    =& \exp\left(
        - \frac{1}{2} \lambda_{\min}(A) d_2
        -\frac{1}{2\alpha} d_p
    \right) \sum_{\vy\not=\vx} \exp\left(
        -\frac{1}{2}U(\vy)+\frac{1}{2}U(\vx)
    \right) \\
    \leq& \exp\left(
        - \frac{1}{2} \lambda_{\min}(A) d_2
        -\frac{1}{2\alpha} d_p
    \right) \sum_{\vy\not=\vx} \exp\left(
        -\frac{1}{2}U(\vy)
    \right) &\justification{assumption that $U(\vx)\leq0$} \\
    \leq& \exp\left(
        - \frac{1}{2} \lambda_{\min}(A) d_2
        -\frac{1}{2\alpha} d_p
    \right) \sum_{\vy\not=\vx} \exp\left(
        -U(\vy)
    \right) &\justification{assumption that $U(\vy)\leq0$} \\
    \leq& \exp\left(
        - \frac{1}{2} \lambda_{\min}(A) d_2
        -\frac{1}{2\alpha} d_p
    \right) \sum_{\vy\in\calX} \exp\left(
        -U(\vy)
    \right) \\
    =& Z \exp\left(
        - \frac{1}{2} \lambda_{\min}(A) d_2
        -\frac{1}{2\alpha} d_p
    \right). &\justification{$\textstyle Z\defeq\sum\nolimits_{\vx\in\calX} \exp(-U(\vx))$} \label{eq:non-self-transition-bound-2}
\end{align}
Combining \cref{eq:non-self-transition-bound-1} and \cref{eq:non-self-transition-bound-2}, we obtain a bound for the non-self-transition probability
\begin{align}
    \label{eq:non-self-transition-bound}
    \sum_{\vy\not=\vx} q_\alpha(\vy\mid\vx) \leq Z \exp\left(
        - \frac{1}{2} \lambda_{\min}(A) d_2
        -\frac{1}{2\alpha} d_p
    \right).
\end{align}
We have established in \cref{thm:pncg-convergence} that the Markov chain defined by $q_\alpha$ is reversible.
It is a well-known fact that the transition matrix of a reversible Markov chain has only real eigenvalues, and hence, the \gershgorin disc theorem (\cref{thm:gershgorin}) in this specific case implies that an eigenvalue $\lambda$ of the transition matrix of $q_{\alpha}$ satisfies
\begin{equation}
    \left|\lambda-q_\alpha(\vx\mid\vx)\right| \leq \sum_{\vy\not=\vx} q_\alpha(\vy\mid\vx) \leq Z \exp\left(
        - \frac{1}{2} \lambda_{\min}(A) d_2
        -\frac{1}{2\alpha} d_p
    \right)
\end{equation}
for at least one of $\vx\in\calX$. In particular, $\lambda_2$, the 2\textsuperscript{nd} largest eigenvalue of the transition matrix of $q_\alpha$, satisfies, for at least one $\vx\in\calX$,
\begin{equation}
    \label{eq:lambda-2-range}
    \left|\lambda_2-q_\alpha(\vx\mid\vx)\right| \leq Z \exp\left(
        - \frac{1}{2} \lambda_{\min}(A) d_2
        -\frac{1}{2\alpha} d_p
    \right).
\end{equation}
In a reversible Markov chain, the \textit{spectral gap} is defined as $\gamma=1-\lambda_2$ \citep[\S12.2]{levin2017}.
Using \cref{eq:non-self-transition-bound} and \cref{eq:lambda-2-range}, we can bound the spectral gap with
\begin{align}
    1-\lambda_2 &= |1-\lambda_2|  &\justification{$1=\lambda_1\geq\lambda_2$ in a reversible transition matrix} \\
    &\leq |1-q_\alpha(\vx\mid\vx)| + |q_\alpha(\vx\mid\vx)-\lambda_2| &\justification{triangle ineq.} \\
    &\leq |1-q_\alpha(\vx\mid\vx)| + Z \exp\left(
        - \frac{1}{2} \lambda_{\min}(A) d_2
        -\frac{1}{2\alpha} d_p
    \right) &\justification{\cref{eq:lambda-2-range}} \\
    &= \sum_{\vy\not=\vx} q_\alpha(\vy\mid\vx)  + Z \exp\left(
        - \frac{1}{2} \lambda_{\min}(A) d_2
        -\frac{1}{2\alpha} d_p
    \right) &\justification{$q_\alpha(\cdot\mid\vx)$ is a distribution} \\
    &\leq 2 \cdot Z \exp\left(
        - \frac{1}{2} \lambda_{\min}(A) d_2
        -\frac{1}{2\alpha} d_p
    \right). &\justification{\cref{eq:non-self-transition-bound}}
\end{align}

Finally, the mixing time and the spectral gap are closely related by the following well-known relationship.
\begin{theorem}[Theorem 12.4 and 12.5 in \citealp{levin2017}]
    \label{thm:mixing-spectral}
    In a reversible, irreducible Markov chain, the spectral gap $\gamma$ and the mixing time $t_{\mix}(\varepsilon)$ are related by
    \begin{equation}
    \left( \frac{1}{\gamma} - 1 \right) \log\left(\frac{1}{2\varepsilon}\right)
    \leq t_\mix(\varepsilon)
    \leq \frac{1}{\gamma} \log\left(\frac{1}{\varepsilon \pi_{\min}}\right)
\end{equation}
where $\pi_{\min}=\min_{x\in\calX} \pi(x)$.
\end{theorem}
Using the left inequality in \cref{thm:mixing-spectral}, we can conclude that
\begin{align}
    t_{\mix}(\varepsilon) &\geq \left( \frac{1}{\gamma} - 1 \right) \log\left(\frac{1}{2\varepsilon}\right) \\
    &=  \left( \frac{1}{1-\lambda_2} - 1 \right) \log\left(\frac{1}{2\varepsilon}\right) \\
    &\geq \left[\frac{1}{2\cdot Z} \exp\left(
        \frac{1}{2} \lambda_{\min}(A) d_2 + \frac{1}{2\alpha} d_p
    \right)-1\right] \log\left(\frac{1}{2\varepsilon}\right) \\
    &= \left[\frac{\exp(\lambda_{\min}(A) d_2/2)}{2\cdot Z} \exp\left(
        \frac{d_p}{2\alpha}
    \right)-1\right] \log\left(\frac{1}{2\varepsilon}\right)
\end{align}
Setting $c_1=\frac{1}{2}\exp(\lambda_{\min}(A) d_2/2)>0$ and $c_2=d_p>0$, we obtain the desired bound
\begin{align}
    t_{\mix}(\varepsilon) \geq \left(\frac{c_1}{Z}\exp\left(\frac{c_2}{2\alpha}\right) - 1 \right) \log\left(\frac{1}{2\varepsilon}\right).
\end{align}
\end{proof}

\section{Experimental Setup}
\label{sec:experimental-setup}

\paragraph{Hyperparameters.}
We found that in GwL, random scan in general performs better than systematic scan. Therefore, all results reported for GwL uses random scan.

\begin{itemize}
    \item \textbf{Toy Example.} In the toy example, the inverse temperature $\beta=0.42$ and the sequence length (i.e., the number of spins in the Ising model) is 5. The underlying Ising topology is a linear chain with the ends connected. All step sizes are tuned with grid search. The step size for \mucola is 1.5, the trajectory length of \svs is $2\pi$, and the step size of \pncg and GwL are both 1.0
    \item \textbf{Sampling from Language Models.} The step size for \mucola is 0.15, and the step size of both \pncg and GwL is 4.0
    \item \textbf{Controlled Generation.} The step size used for \mucola is 1.0 with the energy weight $\beta=2.0$. For \svs and \svs-\langevin, the energy weight is $\beta=1.5$ and the step size is 1.5. Finally, for \pncg and GwL, the step size is $\alpha=1.0$ and the energy is $\beta=1.25$.
\end{itemize}

\paragraph{Classifiers.} We train two classifiers independently, called an internal classifier and an external classifier. The internal classifier is used as the energy function during generation, and the external classifier is used to determine whether the generated text follows the control objective correctly.

The internal classifier is a probing classifier on top of frozen GPT-2 embeddings. The probing classifier is a 3-layered \textsc{BiLSTM} model with $0.5$ dropout. The classifier achieves a $0.84$ F1 score on the test set. We then train an evaluator classifier to evaluate the success rates of the controlled generation algorithms. 

The external classifier is a fine-tuned \textsc{RoBERTa} model that achieves $0.90$ f1-score on the test set.  

\section{Controlled Generation Samples}
\label{sec:control-sample}

We present controlled generation text samples in \cref{tab:control-sample}.

\begin{table*}[t] 
\centering 
\resizebox{\linewidth}{!}{%
\begin{tabular}{@{}ll@{}}\toprule
 \multicolumn{2}{c}{\textbf{Chinese}} \\
 \fudge & In the city centre near Yippee Noodle Bar Chinese, is Alimentum. It has moderate prices and \\ 
 \mucola & and has a 1 out of 5. It has food and high customer rating. The Rice Boat is \\
 \svs-\langevin & It serves Chinese food with a low customer rating. The fast food and restaurant The Golden Curry is a \\
 \svs & It has a low customer rating and a price. The highly rated Chinese restaurant The Phoenix has a high \\
 \pncg + GwL & The Golden Curry is a Chinese food restaurant with a 5 out 5 rating and is not family-friendly \\
 \midrule
 \multicolumn{2}{c}{\textbf{English}} \\
 \fudge & It has an average customer Rating. Bibimbap House has English food in the riverside area near  \\
 \mucola & and has a low customer rating. The Golden Curry is a children friendly, serving English food, with \\
 \svs-\langevin & It has low rating and is located near the to the city centre. The Phoenix is a English food \\
 \svs & Alimentum in the city centre near the a moderate price range. It serves English food, is \\
 \pncg + GwL & Midsummer House serves English food with a moderate price range and a high customer rating. It is \\
 \midrule
 \multicolumn{2}{c}{\textbf{Fast food}} \\
 \fudge & A fast food, coffee shop, Strada has a low customer rating, has a price range of over £30. It is \\
 \mucola & and is family friendly and serves fast food. The Wrestlers is a fast food coffee shop in the \\
 \svs-\langevin & It is located near the riverside, is a cheap family friendly fast food restaurant, and is called \\
 \svs & It is located near the river. The Mill is a cheap, fast food and coffee shop near the \\ 
 \pncg + GwL & Alimentum is a high-priced, child friendly, average rated fast food restaurant that is in \\
 \midrule
 \multicolumn{2}{c}{\textbf{French}} \\
 \fudge & It has a low-priced Inn French food. It is near Café Rouge.The Alimentum is a kid friendly fast food \\
 \mucola & The French restaurant The Waterman is located in the city centre. The price range is less than \\
 \svs-\langevin & It is a restaurant located in the riverside, the restaurant, offers French food with a price \\
 \svs & It is a family restaurant that serves French food with a price range and has a low customer rating. \\ 
 \pncg + GwL & The Waterman, located in city centre, has average French food, is inexpensive and is not family \\
 \midrule
 \multicolumn{2}{c}{\textbf{Indian}} \\
 \fudge & The Phoenix Indian restaurant has moderate prices with a 3 out of 5 rating. Located on the \\
 \mucola & It is in the city and has a low customer rating. The Waterman is a low priced \\
 \svs-\langevin & It is not child friendly and it is near the river. It serves Indian food and a customer rating \\
 \svs & It is located in the city centre near The Portland Arms Indian food and has a low customer rating. \\ 
 \pncg + GwL & The Phoenix is in the city centre that provides Indian food in the cheap price range. Its customer rating \\
 \midrule
 \multicolumn{2}{c}{\textbf{Italian}} \\
 \fudge & It has family Italian food and has a low  a moderate price range. The Rice Boat has an average \\
 \mucola & is a high priced Italian food restaurant with a customer rating of average. The Phoenix is a high \\
 \svs-\langevin & It is located in the city centre, it is not family friendly and is a coffee shop serving Italian \\
 \svs & It is located in the the city centre near The Portland Arms.The Eagle is an Italian restaurant. \\ 
 \pncg + GwL & The Eagle Italian food coffee shop, is a family friendly riverside restaurant with a low customer rating. \\
 \midrule
 \multicolumn{2}{c}{\textbf{Japanese}} \\
 \fudge & Japanese food. Its customer rating is 3 out of 5.The Phoenix is Japanese in the city centre \\
 \mucola & for Japanese food is located in the city centre. It has a low customer rating. The Golden \\
 \svs-\langevin & It is located in the riverside. It is a Japanese food. It is a pub restaurant \\
 \svs & It is located in the riverside. It is a low rated Japanese restaurant, and coffee shop.\\
 \pncg + GwL & It also serves Japanese food. It is located in the city centre and has a high price range. \\
\bottomrule
\end{tabular}}
\caption{Examples of sampled sentences from different control food targets.}
\label{tab:control-sample}
\end{table*}

\end{document}